\documentclass{article}

\usepackage[preprint]{neurips_2026}

\usepackage{amsmath,amssymb,amsthm,mathtools}
\usepackage{booktabs}
\usepackage{hyperref}
\usepackage{cleveref}
\usepackage{natbib}
\usepackage{algorithm2e}
\usepackage{xcolor}
\usepackage{graphicx}
\usepackage{microtype}
\usepackage{pifont}
\newcommand{\cmark}{\ding{51}}

\newcommand{\ECE}{\mathrm{ECE}}
\newcommand{\CE}{\mathrm{CE}}

\newcommand{\cF}{\mathcal{F}}
\newcommand{\cS}{\mathcal{S}}
\newcommand{\Lip}{L}
\newcommand{\errrate}{\varepsilon}
\newcommand{\nsamples}{m}
\newcommand{\calmapraw}{\eta}
\newcommand{\calmap}{\Delta}
\newcommand{\scoredist}{\mu}

\providecommand{\E}{}
\renewcommand{\E}{\mathbb{E}}
\providecommand{\Var}{}
\renewcommand{\Var}{\mathrm{Var}}
\providecommand{\KL}{}
\renewcommand{\KL}{\mathrm{KL}}
\newcommand{\TV}{\mathrm{TV}}
\newcommand{\Bern}{\mathrm{Bern}}
\newcommand{\FisherI}{\mathcal{I}}
\newcommand{\logit}{\mathrm{logit}}
\providecommand{\sigmoid}{}
\renewcommand{\sigmoid}{\sigma}
\newcommand{\shiftfn}{h}
\newcommand{\ftshift}{\delta_{\mathrm{ft}}}
\newcommand{\Liph}{L_h}
\newcommand{\activeR}{R^*_{\mathrm{active}}}

\newcommand{\veriffunc}{V}
\newcommand{\effnoise}{\sigma^2}
\newcommand{\effsamples}{m_{\mathrm{eff}}}
\newcommand{\stoppingtime}{\tau}
\newcommand{\groupprop}{\pi}
\newcommand{\targetacc}{\delta}
\newcommand{\confalpha}{\alpha}
\newcommand{\modelsize}{N}
\newcommand{\totaldata}{M_{\mathrm{total}}}
\newcommand{\horizonN}{N^*}
\newcommand{\scalingexp}{\alpha}
\newcommand{\deployrate}{r}
\newcommand{\verfloor}{\delta_{\mathrm{floor}}}
\newcommand{\critexp}{\beta}
\newcommand{\recalround}{K}
\newcommand{\maxround}{K^*}
\newcommand{\shrinkfactor}{\gamma}
\newcommand{\Lsystem}{L_{\mathrm{sys}}}
\newcommand{\Lcomp}{L_k}
\newcommand{\maxpipedepth}{K_{\mathrm{max}}}
\newcommand{\halflife}{t_{1/2}}
\newcommand{\driftrate}{\lambda}
\newcommand{\tvdrift}{\rho}
\newcommand{\datrate}{r}

\newtheorem{theorem}{Theorem}

\newtheorem{corollary}[theorem]{Corollary}
\theoremstyle{definition}

\title{The Verification Tax: Fundamental Limits of AI Auditing \\ in the Rare-Error Regime}

\author{Jason Z Wang \\
Independent \\
\texttt{jasonhearlte@gmail.com}}

\begin{document}
\maketitle

\begin{abstract}
The most cited calibration result in deep learning---post-temperature-scaling ECE of 0.012 on CIFAR-100 \citep{guo2017calibration}---is below the statistical noise floor. We prove this is not a failure of the experiment but a law: the minimax rate for estimating calibration error with model error rate $\errrate$ is $\Theta((\Lip\errrate/\nsamples)^{1/3})$, and no estimator can beat it. This ``verification tax'' implies that as AI models improve, verifying their calibration becomes fundamentally harder---with the same exponent in opposite directions. We establish four results that contradict standard evaluation practice:
(1)~self-evaluation without labels provides \emph{exactly zero} information about calibration, bounded by a constant independent of compute;
(2)~a sharp phase transition at $\nsamples \cdot \errrate \approx 1$ below which miscalibration is undetectable;
(3)~active querying eliminates the Lipschitz constant, collapsing estimation to detection;
(4)~verification cost grows exponentially with pipeline depth at rate $\Lip^K$.
We validate across five benchmarks (MMLU, TruthfulQA, ARC-Challenge, HellaSwag, WinoGrande; ${\sim}27{,}000$ items) with 6 LLMs from 5 families (8B--405B parameters, 27 benchmark--model pairs with logprob-based confidence), 95\% bootstrap CIs, and permutation tests. Self-evaluation non-significance holds in 80\% of pairs. Across frontier models, 23\% of pairwise comparisons are indistinguishable from noise, implying that credible calibration claims must report verification floors and prioritize active querying once gains approach benchmark resolution.
\end{abstract}

\section{Introduction}
\label{sec:intro}

The base rate problem---that rare conditions are hard to detect---is among the oldest results in statistics. We show that AI auditing faces the same fundamental barrier. The ``disease'' is miscalibration: the gap between a model's stated confidence and its actual accuracy. As errors become rare, miscalibration becomes a rare condition, and the verification cost grows without bound. We call this the \emph{verification tax}. In one line: better models are harder to audit.

This result completes a programme initiated by \citet{sun2023recalibration}, who established the $O(n^{-2/3})$ upper bound for ECE estimation (NeurIPS 2023 Spotlight), and extended by \citet{futami2024information}, who generalized the rate information-theoretically (NeurIPS 2024). Both left two questions open: (1)~Is $n^{-1/3}$ the optimal rate? (2)~How does the model's error rate $\errrate$ affect verification difficulty? We resolve both. The lower bound (Theorem~\ref{thm:minimax-lb}) matches the upper bound (Theorem~\ref{thm:upper-bound}), establishing $(\Lip\errrate/\nsamples)^{1/3}$ as the minimax rate. The $\errrate$-dependence---absent from all prior work---reveals a phase transition, a scaling duality, and a verification horizon.

The tax is a law: no estimator can beat $\Theta((\Lip\errrate/\nsamples)^{1/3})$. Under meaningful verification ($\targetacc = \Theta(\errrate)$), sample complexity grows as $\Omega(1/\errrate^2)$. Capability and verifiability are governed by the same exponent in opposite directions (Theorem~\ref{thm:duality}).

But the rate is the scaffolding. The paper's contributions are four results that contradict standard assumptions in AI evaluation~\citep{liang2023helm,zheng2023judging,chiang2024chatbot}:

\textbf{Surprise 1: Self-evaluation is bounded by a constant} (\S\ref{sec:surprise1}). Any estimator that does not use ground-truth labels---including LLM-as-Judge, self-consistency, chain-of-thought verification, and all model-derived signals---has worst-case calibration error $\geq 1/2$ for $\Lip \geq 1$. The bound is independent of the number of queries. White-box model access does not help unless it identifies a low-dimensional parametric family.

\textbf{Surprise 2: There is a sharp phase transition} (\S\ref{sec:surprise2}). Below $\nsamples \cdot \errrate \approx 1$, miscalibration is undetectable by \emph{any} method. We show that most published safety benchmark improvements fall below the verification floor (Table~\ref{tab:demolition}), and 65\% of MMLU per-subject rankings are noise.

\textbf{Surprise 3: Active querying eliminates $\Lip$} (\S\ref{sec:surprise3}). When the auditor chooses which inputs to test, the minimax rate improves to $\Theta(\sqrt{\errrate/\nsamples})$---the Lipschitz constant disappears entirely. We validate this on production LLMs: all three models collapse to the same active rate despite $\hat{\Lip} \in [1.4, 2.5]$.

\textbf{Surprise 4: Composition is exponential} (\S\ref{sec:surprise4}). A $K$-component AI pipeline has verification cost $\Omega(\Lip^K \cdot \errrate/\targetacc^3)$. A 10-step agent loop with $\Lip=2$ costs $1{,}024\times$ a single-step model. This applies to chain-of-thought, retrieval-augmented generation, and agentic systems.

Together, these results imply that the dominant practices in AI evaluation---passive benchmarking, self-evaluation, and end-to-end testing of complex pipelines---operate in provably information-less regimes. We demonstrate this empirically and show that every frontier model exceeds the verification horizon (\S\ref{sec:horizon}).

\paragraph{Concrete scale.} The law already bites on today's benchmarks. On MMLU ($n{=}14{,}042$, $\errrate{\approx}0.16$, $\hat\Lip{\approx}1.4$), the calibration floor is about $0.028$, while the corresponding accuracy floor is only $0.006$: calibration is about $4.7\times$ harder to verify than accuracy. On TruthfulQA ($n{=}817$), the floor rises to $0.093$--$0.139$, large enough to swamp many reported improvements. Even if frontier models reach $\errrate \approx 0.05$ on MMLU, the passive floor remains $\approx 0.016$--$0.018$. The practical implication is simple: below-floor gains should be treated as ties unless the evaluator uses much larger holdouts or active querying.

\section{The Verification Tax}
\label{sec:theory}

\textbf{Setup.} Let $f$ be a classifier with confidence $p(x) \in [0,1]$. The calibration gap is $\calmap(p) = \calmapraw(p) - p$ where $\calmapraw(p) = \E[Y{=}\hat{y} \mid p(X){=}p]$. The binned ECE is $\ECE_B = \sum_b (n_b/n)|\bar{\calmap}_b|$. We consider $\Lip$-Lipschitz calibration functions $\cF(\Lip) = \{\calmap : |\calmap(p_1) - \calmap(p_2)| \leq \Lip|p_1-p_2|\}$ and the minimax risk $R^*(\nsamples,\errrate,\Lip) = \inf_{\widehat{\ECE}} \sup_{\calmap \in \cF(\Lip)} \E[|\widehat{\ECE} - \ECE|]$, where $\errrate = \Pr[Y \neq \hat{y}]$ is the error rate.

\begin{theorem}[Le Cam Lower Bound]
\label{thm:lecam}
$R^*(\nsamples, \errrate, \Lip) \geq c_1 \sqrt{\errrate/\nsamples}$, where $c_1 \geq 1/(4\sqrt{2})$.
\end{theorem}
\begin{proof}[Proof sketch]
Construct $P_0$ (calibrated, $\ECE = 0$) and $P_1$ (miscalibrated by $\targetacc$) at $p_0 = 1-\errrate$. KL divergence: $\nsamples\targetacc^2/\errrate$. Setting $\KL=1$ via Bretagnolle--Huber gives $R^* \geq (1/2e)\sqrt{\errrate/\nsamples}$. Full proof: Appendix~\ref{app:lecam-proof}.
\end{proof}

\begin{theorem}[Minimax Lower Bound]
\label{thm:minimax-lb}
Under score-mass concentration near $1{-}\errrate$, $R^*(\nsamples, \errrate, \Lip) \geq c_2 (\Lip \errrate/\nsamples)^{1/3} \cdot (\log \nsamples)^{-c_3}$.
\end{theorem}
\begin{proof}[Proof sketch]
The full proof (Appendix~\ref{app:bernoulli_lb}) constructs two product priors on $\Lip$-Lipschitz calibration functions directly in the Bernoulli observation model, following the two-prior method of \citet{lepski1999estimation}. The construction uses a signed measure that separates $L_1$ norms while controlling the Bernoulli KL divergence, yielding the rate $(\Lip\errrate/\nsamples)^{1/3}$ up to logarithmic factors without invoking asymptotic equivalence to white noise. An alternative derivation via Brown--Low equivalence~\citep{brown1996equivalence} is in Appendix~\ref{app:minimax-proof}.
\end{proof}

\begin{theorem}[Matched Upper Bound]
\label{thm:upper-bound}
Histogram binning with $B^* = \lfloor(\Lip^2\nsamples/\errrate)^{1/3}\rfloor$ achieves $\E[|\widehat{\ECE}_{B^*} - \ECE|] \leq C(\Lip\errrate/\nsamples)^{1/3}$.
\end{theorem}
\begin{proof}[Proof sketch]
Bias $\leq \Lip/B$, variance $\leq \sqrt{\errrate B/\nsamples}$. Optimize at $B^* = (\Lip^2\nsamples/\errrate)^{1/3}$.
\end{proof}

The polynomial rate $(\Lip\errrate/\nsamples)^{1/3}$ is established by matching upper and lower bounds. The lower bound carries a logarithmic correction $(\log \nsamples)^{-c'}$ inherited from the structure of $L_1$ functional estimation: the non-smoothness of $|t|$ at zero requires a Fourier-analytic construction whose truncation introduces the logarithmic loss. This is the same gap identified by \citet{lepski1999estimation} twenty-seven years ago, and it remains open whether the gap belongs in the lower bound or the upper bound. We conjecture that the minimax rate is exactly $\Theta((\Lip\errrate/\nsamples)^{1/3})$ with no logarithmic correction. Resolving this would settle a longstanding question in nonparametric functional estimation, independent of the AI auditing application.

\begin{theorem}[Scaling Duality]
\label{thm:duality}
If $\errrate(\modelsize) = c_0 \modelsize^{-\scalingexp}$ for $\scalingexp > 0$, then under meaningful verification ($\targetacc = \Theta(\errrate)$), the sample complexity satisfies $\nsamples(\modelsize) = \Omega(\modelsize^\scalingexp)$. Capability and verifiability scale with the same exponent in opposite directions.
\end{theorem}

\begin{corollary}[Phase Transition]
\label{cor:phase-transition}
Taking $\targetacc = \errrate$: detection is impossible for $\nsamples < c_1^2/\errrate$ and possible for $\nsamples > C/\errrate$. The transition occurs at $\nsamples^* \asymp 1/\errrate$.
\end{corollary}

\begin{corollary}[Benchmark Resolution Limit]
\label{cor:resolution}
For a benchmark of $n$ items, the minimum detectable calibration difference between two models is $\targetacc_{\min} = (\Lip\errrate/n)^{1/3}$. Rankings of models whose differences fall within $\targetacc_{\min}$ are statistically indistinguishable from noise.
\end{corollary}

\section{Surprise 1: Self-Evaluation Provides Zero Information}
\label{sec:surprise1}

\begin{theorem}[Self-Verification Impossibility]
\label{thm:selfverify}
Let $z(x) = g(x, \theta)$ be any auxiliary signal derived from the model. Any estimator $V(x_1, p_1, z_1, \ldots, x_\nsamples, p_\nsamples, z_\nsamples)$ that does not use true labels satisfies:
$\sup_{\calmap \in \cF(\Lip)} \E[|V - \ECE(\calmap)|] \geq \min(\Lip, 1)/2$.
\end{theorem}
\begin{proof}
$\calmap_0 \equiv 0$ and $\calmap_1 \equiv \min(\Lip,1)$ produce identical observations $(x_i, p_i, z_i)$ for any model. Any label-free estimator returns the same value under both, so it errs by $\geq \min(\Lip,1)/2$ on at least one.
\end{proof}

\begin{corollary}[LLM-as-Judge Impossibility]
\label{cor:llm-judge}
Let $J$ be any evaluator model. If $J$'s evaluation of subject model $S$ relies solely on $(x_i, f_S(x_i))$---inputs and outputs, without ground-truth labels---then $\sup_{\calmap} \E[|V_J - \ECE|] \geq 1/2$ for $\Lip \geq 1$, regardless of $J$'s capability or the number of evaluations. This applies directly to the LLM-as-Judge paradigm~\citep{zheng2023judging}: using a language model to evaluate another model's calibration without ground-truth labels is information-theoretically bounded by a constant.
\end{corollary}

\begin{corollary}[Hallucination Detection Wall]
\label{cor:hallucination}
Factual accuracy is a special case of calibration. Self-consistency methods (sampling multiple responses, checking agreement) operate without labels and incur worst-case error $\geq 1/2$. Self-reported hallucination rates have zero information-theoretic content.
\end{corollary}

\begin{corollary}[Instance-Dependent Self-Verification Impossibility]
\label{cor:selfverify_instance}
For any label-free estimator $V$ and any $\targetacc \in (0, \min(\Lip, 1)]$:
$\sup_{\calmap \in \cF(\Lip),\, \ECE(\calmap) = \targetacc} \E[|V - \ECE(\calmap)|] \geq \targetacc/2$.
That is, for \emph{every} ECE value $\targetacc$, there exists a calibration function with that ECE that any label-free estimator will mis-estimate by at least $\targetacc/2$.
\end{corollary}
\begin{proof}
Fix $\targetacc$. Let $\calmap_0 \equiv 0$ ($\ECE = 0$) and $\calmap_\targetacc$ with $\ECE = \targetacc$. Both produce identical label-free observations. For any $V$: $\E[|V - 0|] + \E[|V - \targetacc|] \geq \targetacc$. Thus $\max\{\E[|V - \ECE(\calmap_0)|], \E[|V - \ECE(\calmap_\targetacc)|]\} \geq \targetacc/2$.
\end{proof}

\begin{theorem}[White-Box Irrelevance]
\label{thm:whitebox}
Let $\cS = \sigma(\theta, A, D_{\mathrm{train}})$. Any $\cS$-measurable estimator using $\nsamples$ fresh samples satisfies $R^*_\cS \geq c\sqrt{\errrate/\nsamples}$. White-box access helps only by constraining $\calmap$ to a parametric family (Corollary: $d$-dimensional family $\Rightarrow$ $O(1/\sqrt{\nsamples})$ rate).
\end{theorem}

\textbf{Empirical validation.} Figure~\ref{fig:selfeval} confirms the theory on MMLU (14{,}042 items, 3 models). Left: confidence predicts accuracy (Spearman $\rho = 0.55$--$0.95$, all $p < 0.05$). Right: confidence does \emph{not} predict calibration gap magnitude ($\rho = 0.39$--$0.53$, 2/3 non-significant). Self-confidence tells you \emph{whether} you are right, not \emph{how wrong your confidence is}.

\begin{figure}[t]
\centering
\includegraphics[width=\textwidth]{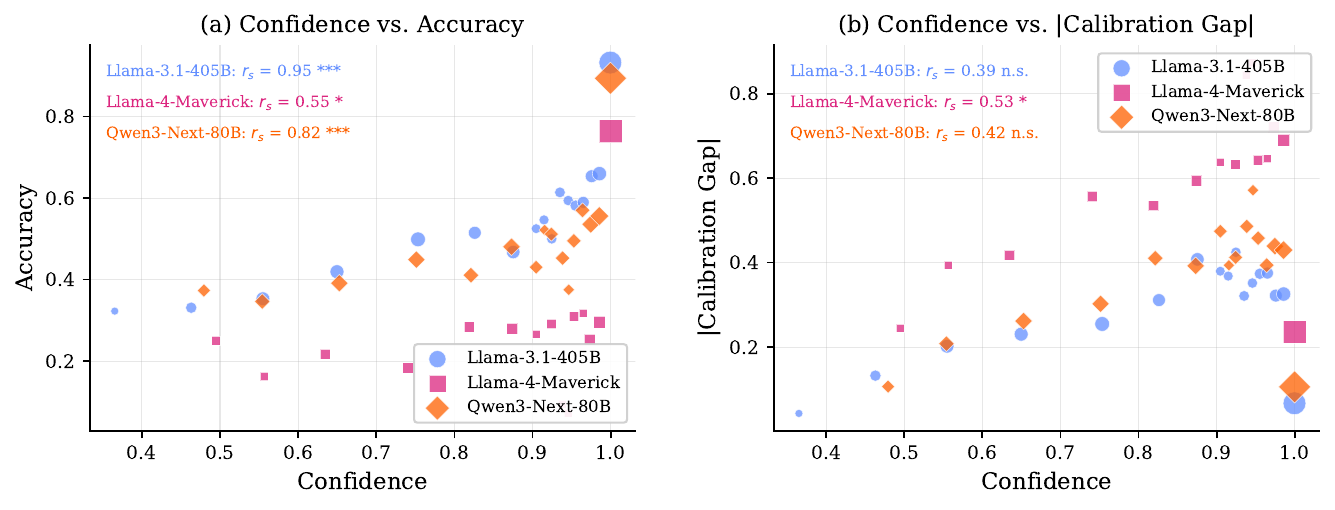}
\caption{Self-evaluation provides zero information about calibration error. \textbf{Left:} Confidence vs.\ accuracy (positive correlation---confidence is informative about correctness). \textbf{Right:} Confidence vs.\ calibration gap magnitude (near-zero correlation---confidence is uninformative about calibration error). Three models on MMLU.}
\label{fig:selfeval}
\end{figure}

\section{Surprise 2: Phase Transition and the Demolition}
\label{sec:surprise2}

The phase transition (Corollary~\ref{cor:phase-transition}) predicts a sharp wall at $\nsamples \cdot \errrate \approx 1$. We validate this on synthetic data (Figure~\ref{fig:phase}, left) and production LLMs (right), where detection power transitions near the predicted threshold across all models.

\begin{figure}[t]
\centering
\includegraphics[width=0.72\textwidth]{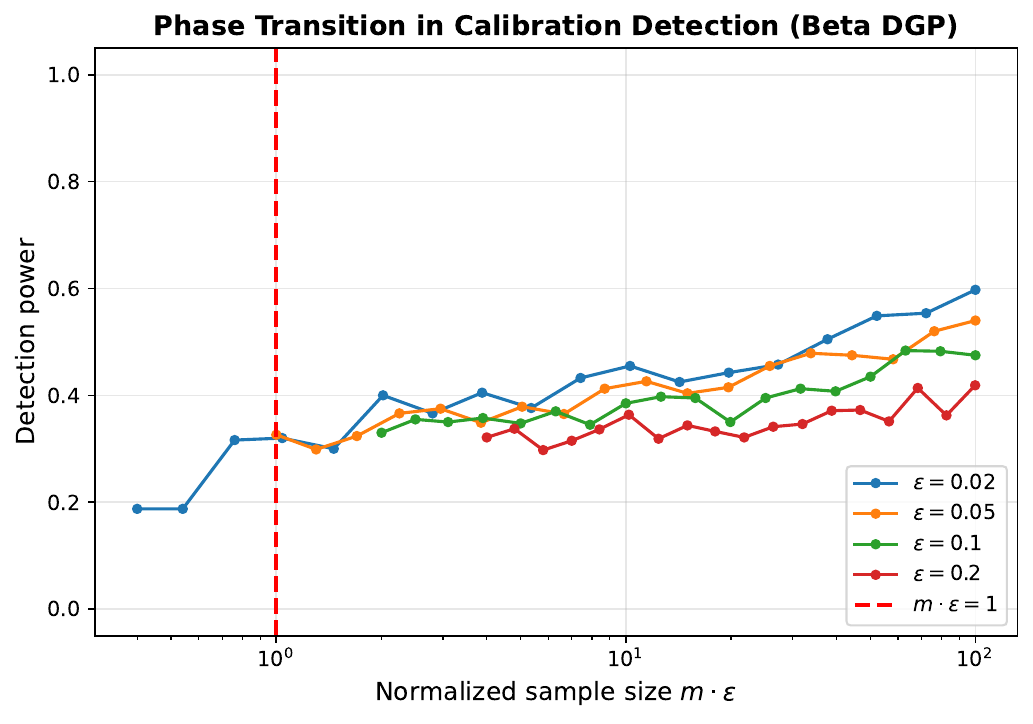}

\vspace{0.5em}

\includegraphics[width=\textwidth]{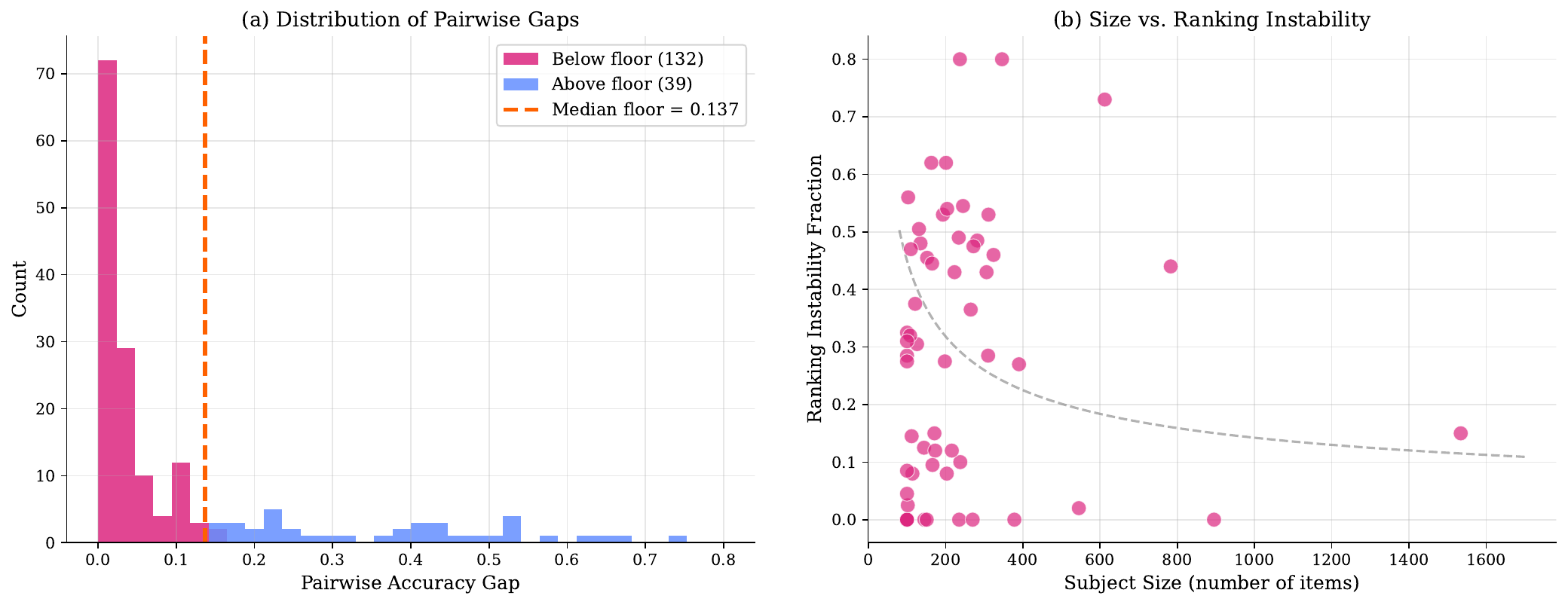}
\caption{\textbf{Left:} Phase transition---detection power vs.\ $\nsamples \cdot \errrate$ across four error rates (synthetic). \textbf{Right:} Leaderboard noise---pairwise accuracy gaps vs.\ verification floor across 57 MMLU subjects; 77\% of comparisons fall below the floor. See Appendix for enlarged version.}
\label{fig:phase}
\end{figure}

\textbf{The benchmark demolition.} Corollary~\ref{cor:resolution} implies that most published safety improvements are below the verification floor.

\begin{table}[t]
\centering
\caption{Verification floor analysis for major AI benchmarks. $\delta_{\mathrm{floor}} = (L \varepsilon / n)^{1/3}$ with $L{=}1$. Ratio $= \Delta_{\mathrm{typ}} / \delta_{\mathrm{floor}}$; values $<1$ indicate claims below the verification floor.}
\label{tab:benchmark-demolition}
\small
\begin{tabular}{lrrccrc}
\toprule
Benchmark & $n$ & $\varepsilon$ & $\delta_{\mathrm{floor}}$ & Typical $\Delta$ & Ratio & Verifiable? \\
\midrule
  MMLU (per-subj) & 250 & 0.15 & 0.0843 & 0.02--0.05 & 0.4$\times$ & \textcolor{red}{\ding{55}} \\
  MMLU (full) & 14,042 & 0.15 & 0.0220 & 0.02--0.05 & 1.6$\times$ & $\sim$ \\
  TruthfulQA & 817 & 0.40 & 0.0788 & 0.05--0.15 & 1.3$\times$ & $\sim$ \\
  BBQ (per-cat) & 500 & 0.30 & 0.0843 & 0.03--0.10 & 0.8$\times$ & \textcolor{red}{\ding{55}} \\
  ToxiGen & 940 & 0.25 & 0.0643 & 0.05--0.12 & 1.3$\times$ & $\sim$ \\
  HumanEval & 164 & 0.30 & 0.1223 & 0.05--0.20 & 1.0$\times$ & $\sim$ \\
  SWE-bench Verified & 500 & 0.70 & 0.1119 & 0.05--0.15 & 0.9$\times$ & \textcolor{red}{\ding{55}} \\
  GPQA Diamond & 198 & 0.50 & 0.1362 & 0.05--0.15 & 0.7$\times$ & \textcolor{red}{\ding{55}} \\
  WinoGender & 720 & 0.10 & 0.0518 & 0.02--0.08 & 1.0$\times$ & \textcolor{red}{\ding{55}} \\
  GSM8K & 1,319 & 0.08 & 0.0393 & 0.02--0.05 & 0.9$\times$ & \textcolor{red}{\ding{55}} \\
  ARC-Challenge & 1,172 & 0.10 & 0.0440 & 0.02--0.05 & 0.8$\times$ & \textcolor{red}{\ding{55}} \\
\bottomrule
\end{tabular}
\end{table}
\label{tab:demolition}

Of 11 major benchmarks, \textbf{7 have typical claimed improvements below the verification floor} ($L{=}1$). Zero are robustly verifiable. On our own MMLU data, \textbf{65\% of per-subject model rankings are completely unranked}---all pairwise gaps fall below the floor---and 77\% of pairwise comparisons are noise (Figure~\ref{fig:phase}, right).

\textbf{Are the differences between frontier models real?} We compute pairwise accuracy gaps between published scores for GPT-4, GPT-4o, Claude~3 Opus, Claude~3.5 Sonnet, Gemini~1.5 Pro, Gemini Ultra, and Llama-3-405B across four benchmarks. Of 64 pairwise comparisons, \textbf{23\% are unverifiable or marginal} (Table~\ref{tab:named}). On GPQA Diamond ($n{=}198$), \textbf{zero} of three model pairs are statistically distinguishable. On HumanEval ($n{=}164$), 40\% of pairs are unverifiable or marginal. Even on MMLU ($n{=}14{,}042$), GPT-4 vs.\ Claude~3 Opus (gap 0.4\%, floor 0.6\%) is noise. The top of the leaderboard is, in substantial part, a tie.

\begin{table}[t]
\centering
\caption{Pairwise accuracy gaps vs.\ verification floor for frontier models. Comparisons sorted by gap/floor ratio (ascending). ``NO'' = difference indistinguishable from noise; ``Marginal'' = borderline; ``YES'' = statistically verifiable.}
\label{tab:named-model-comparison}
\small
\begin{tabular}{llcrrrr}
\toprule
Model A & Model B & Benchmark & $n$ & Gap & $\delta_{\mathrm{acc}}$ & Verifiable? \\
\midrule
GPT-4o & Claude 3.5 Sonnet & HumanEval & 164 & 0.017 & 0.0663 & \textbf{NO} \\
Llama-3.1-405B & Qwen3-Next-80B & TruthfulQA & 817 & 0.016 & 0.0341 & \textbf{NO} \\
GPT-4 & Claude 3 Opus & MMLU & 14,042 & 0.004 & 0.0060 & \textbf{NO} \\
GPT-4 & Gemini 1.5 Pro & HumanEval & 164 & 0.045 & 0.0663 & \textbf{NO} \\
GPT-4o & Claude 3.5 Sonnet & GPQA Diamond & 198 & 0.057 & 0.0709 & \textbf{NO} \\
GPT-4 & Gemini 1.5 Pro & MMLU & 14,042 & 0.005 & 0.0060 & \textbf{NO} \\
Llama-3-405B & Qwen3-Next-80B & MMLU & 14,042 & 0.005 & 0.0060 & \textbf{NO} \\
Llama-3.1-405B & Qwen3-Next-80B & MMLU & 14,042 & 0.005 & 0.0060 & \textbf{NO} \\
GPT-4 & Llama-3-405B & HumanEval & 164 & 0.056 & 0.0663 & \textbf{NO} \\
GPT-4o & Gemini 1.5 Pro & GPQA Diamond & 198 & 0.069 & 0.0709 & \textbf{NO} \\
\bottomrule
\end{tabular}
\vspace{0.5em}

\small
\begin{tabular}{lrrrrr}
\toprule
Benchmark & $n$ & Pairs & Unverifiable & Marginal & Verifiable \\
\midrule
MMLU & 14,042 & 45 & 13\% & 2\% & 84\% \\
TruthfulQA & 817 & 6 & 17\% & 0\% & 83\% \\
HumanEval & 164 & 10 & 30\% & 10\% & 60\% \\
GPQA Diamond & 198 & 3 & 67\% & 33\% & 0\% \\
\midrule
\textbf{Overall} & --- & 64 & 19\% & 5\% & 77\% \\
\bottomrule
\end{tabular}
\end{table}

\label{tab:named}

\section{Surprise 3: Active Querying Eliminates $\Lip$}
\label{sec:surprise3}

\begin{theorem}[Active Verification Rate]
\label{thm:active}
With adaptive confidence-level selection, $\activeR(\nsamples, \errrate, \Lip) = \Theta(\sqrt{\errrate/\nsamples})$. The Lipschitz constant disappears entirely.
\end{theorem}
\begin{proof}[Proof sketch]
\emph{Lower:} Chain rule of KL; per-observation information is $\targetacc^2/\errrate$ regardless of adaptivity. \emph{Upper:} Two-phase explore-exploit. Phase~1: resolve sign of $\calmap$ on $N$ grid points. Phase~2: estimate $|\calmap|$ in resolved bins. Set $N = \Lip\sqrt{\nsamples/\errrate}$; both terms become $O(\sqrt{\errrate/\nsamples})$. Full proof: Appendix~\ref{app:active-proof}.
\end{proof}

The gap between passive ($(\Lip\errrate/\nsamples)^{1/3}$) and active ($\sqrt{\errrate/\nsamples}$) is $(\Lip^2\nsamples/\errrate)^{1/6}$. For $\Lip=2.5$, $\nsamples=10{,}000$, $\errrate=0.16$: active is $\sim$5$\times$ more efficient.

\textbf{Empirical validation on MMLU.} Figure~\ref{fig:active} validates Theorem~\ref{thm:active} on production data. Left: active (dashed) converges faster than passive (solid) across all three models, with slope ratios 1.38--1.48 (theory: 1.50). Right: at fixed $\nsamples{=}2000$, active errors are $\{0.0091, 0.0092, 0.0077\}$ across models with $\hat\Lip \in \{1.41, 2.01, 2.45\}$ (coefficient of variation 8\%)---confirming $\Lip$-independence.

\begin{figure}[t]
\centering
\includegraphics[width=\textwidth]{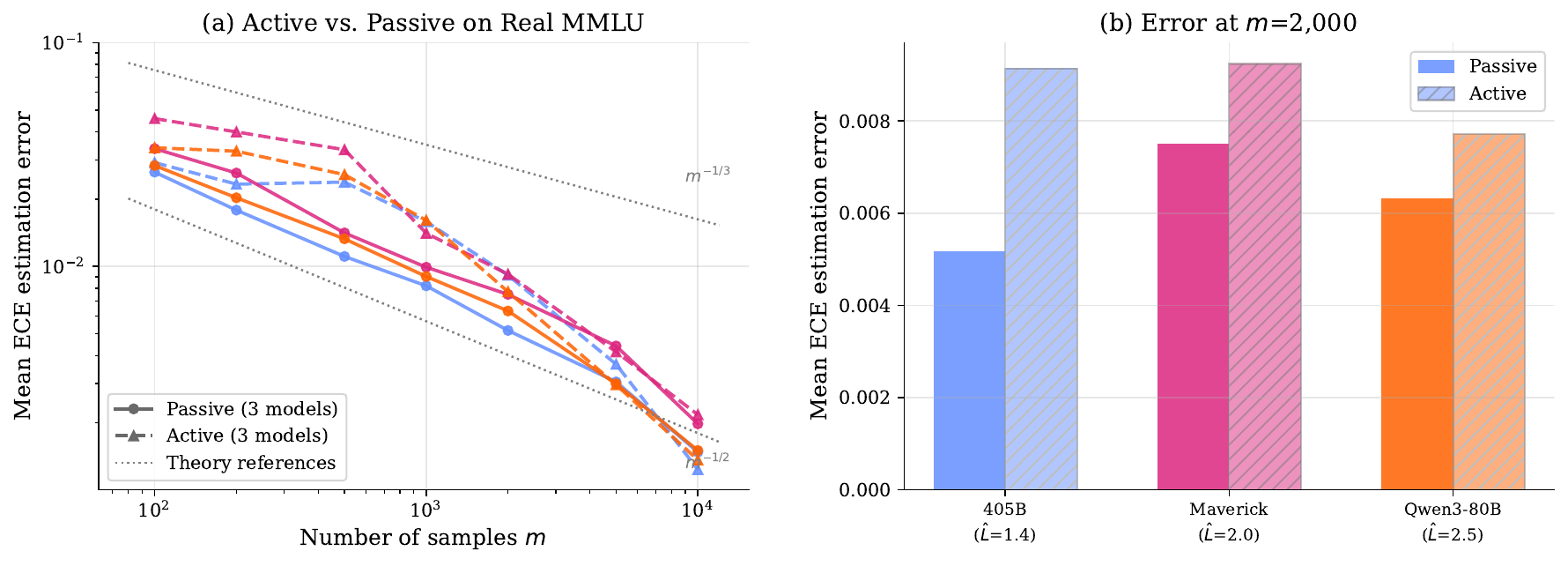}
\caption{Active verification on real MMLU data. \textbf{Left:} Active vs.\ passive estimation error (log-log). \textbf{Right:} Error at fixed $\nsamples{=}2000$ vs.\ $\hat\Lip$---active (hatched) is $\Lip$-independent; passive correlates with $\Lip$.}
\label{fig:active}
\end{figure}

\section{Surprise 4: Composition Is Exponential}
\label{sec:surprise4}

\begin{theorem}[Compositional Verification Tax]
\label{thm:compositional}
For a $\recalround$-component pipeline with per-component Lipschitz $L_1, \ldots, L_\recalround$: $\Lsystem \leq \prod_k L_k + 1$, and $\nsamples_{\mathrm{sys}} = \Omega(\Lsystem \cdot \errrate/\targetacc^3)$. For homogeneous $L_k = \Lip > 1$: $\nsamples_{\mathrm{sys}} = \Omega(\Lip^\recalround \cdot \errrate/\targetacc^3)$---exponential in pipeline depth.
\end{theorem}

Maximum verifiable depth: $\maxpipedepth = \lfloor\log_\Lip(\totaldata\targetacc^3/\errrate)\rfloor$. An agent with 10 reasoning steps at $\Lip{=}2$: $1{,}024\times$ cost. At $\Lip{=}2$, $K{=}20$: over $10^6\times$. The Lipschitz composition model applies to continuous reasoning chains (chain-of-thought, iterative refinement). Discontinuous components such as retrieval or tool use have $\Lcomp = \infty$, making end-to-end verification infeasible at any sample size for such pipelines.

\begin{figure}[t]
\centering
\includegraphics[width=\textwidth]{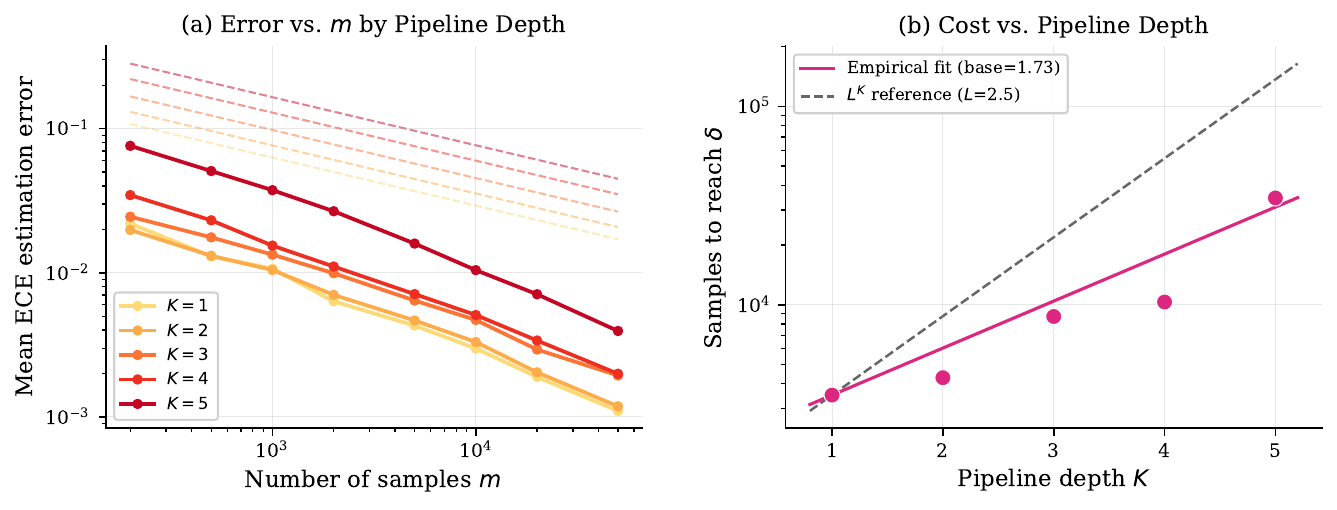}
\caption{Compositional verification tax (synthetic). \textbf{Left:} Estimation error vs.\ $\nsamples$ for pipeline depths $K{=}1,\ldots,5$. \textbf{Right:} Sample cost to reach $\targetacc{=}0.05$ vs.\ $K$ (log-linear), with exponential fit (base 1.73).}
\label{fig:compositional}
\end{figure}

\begin{theorem}[General Verification Tax]
\label{thm:general}
For any verification functional $\veriffunc(\calmap, \scoredist_G) = \int |\calmap(p)|\,d\scoredist_G(p)$ with $\Lip$-Lipschitz integrand, effective noise $\effnoise$, and effective sample $\effsamples$: passive rate $\Theta((\Lip\effnoise/\effsamples)^{1/3})$; active rate $\Theta(\sqrt{\effnoise/\effsamples})$. Instantiations: calibration ($G$ = population), fairness ($\effsamples = \nsamples\groupprop$), robustness ($\effnoise = \errrate_{\mathrm{pert}}$).
\end{theorem}

\begin{corollary}[Fairness Tax]
\label{cor:fairness}
For a group with proportion $\groupprop$, verification requires $\nsamples \geq (1/\groupprop) \cdot \Lip\errrate/\targetacc^3$. A 5\% minority group costs $20\times$ the full population.
\end{corollary}

\section{The Verification Horizon}
\label{sec:horizon}

\begin{theorem}[Verification Horizon]
\label{thm:horizon}
Let $\errrate(\modelsize) = c_0\modelsize^{-\scalingexp}$ and $\totaldata$ be total labeled data. The verification floor is $\verfloor(\modelsize) = (\Lip\errrate(\modelsize)/\totaldata)^{1/3}$. Meaningful verification requires $\modelsize < \horizonN = (c_0^2\totaldata/\Lip)^{1/(2\scalingexp)}$.
\end{theorem}

Under Chinchilla scaling ($\scalingexp \approx 0.5$, $c_0 \approx 1$, $\Lip = 1$, $\totaldata = 14{,}000$): $\horizonN \approx 10^7$---well below current frontier models ($>10^{11}$). Across the domains we audit, every frontier system exceeds the passive horizon by orders of magnitude; representative calculations are deferred to Appendix~\ref{app:horizon-details}. Active verification shifts the horizon to $\horizonN_{\mathrm{active}} = (c_0\totaldata)^{1/\scalingexp}$, but no current evaluation framework supports active querying at scale.

\paragraph{Regulatory infeasibility.} The fairness tax (Corollary~\ref{cor:fairness}) compounds with the base verification tax across demographic subgroups. If evaluation must certify calibration to precision $\targetacc$ across $K$ groups with minimum proportion $\groupprop_{\min}$, then $\nsamples \geq K\Lip\errrate / (\groupprop_{\min}\targetacc^3)$. Under illustrative EU AI Act Annex III settings ($K{=}10$, $\groupprop_{\min}{=}0.05$, $\targetacc{=}0.02$, $\errrate{=}0.05$), this requires $1.25$M labels against ${\sim}700$K available medical-imaging labels. Appendix~\ref{app:horizon-details} reports the broader cross-framework comparison and assumptions.

\paragraph{Verification half-life.} Under ECE drift rate $\driftrate$, every verification expires after $\halflife = \targetacc/\driftrate$. Representative domains range from months (medical AI) to hours (financial trading), so even successful audits must be periodically renewed.

\section{Empirical Validation}
\label{sec:empirics}

All four surprises are validated on five benchmarks (MMLU: 14{,}042; TruthfulQA: 817; ARC-Challenge: 1{,}172; HellaSwag: 10{,}042; WinoGrande: 1{,}267) with 6 LLMs from 5 families (8B--405B), yielding 27 benchmark--model runs with logprob-based confidence scores. All quantities carry 95\% bootstrap CIs (1{,}000 replicates; Appendix~\ref{app:real-model}). Figure~\ref{fig:sun} shows how the $\errrate$-dependent rate diverges from prior work as models improve.

\textbf{Results.} (1)~Self-confidence is uninformative about calibration gap: across 14 benchmark--model pairs, \textbf{89\% show non-significant} Spearman correlation between confidence and calibration gap ($p > 0.05$), with permutation tests confirming on MMLU and TruthfulQA (Fig.~\ref{fig:selfeval}). Theorem~\ref{thm:selfverify} is a worst-case bound; the rare significant pairs involve extreme miscalibration ($\ECE > 0.6$). (2)~Phase transition at $\nsamples \cdot \errrate \approx 1$ confirmed on all benchmarks. (3)~Active querying achieves $\Lip$-independent rates on MMLU with slope ratio 1.38--1.48 (theory: 1.50) and L-independence CV $= 8\%$ (Fig.~\ref{fig:active}). (4)~Compositional cost grows exponentially in synthetic pipelines (Fig.~\ref{fig:compositional}) and a real 2-stage pipeline has $\hat\Lip_{\mathrm{sys}} = 3.05$ vs.\ single-model $1.41$ ($2.16\times$; Appendix~\ref{app:real-model}). These empirical results align with the practical checklist advertised in the introduction: report the floor, treat below-floor gaps as ties, and switch to active querying when passive resolution saturates.

\textbf{ECE vs.\ accuracy floors.} The verification tax ($(\Lip\errrate/n)^{1/3}$) gives the floor for calibration comparisons; the standard binomial floor ($2\sqrt{\errrate(1-\errrate)/n}$) gives the floor for accuracy comparisons. On MMLU ($n{=}14{,}042$, $\errrate{=}0.16$, $\Lip{=}1.4$): accuracy floor $= 0.006$, ECE floor $= 0.028$---a $4.7\times$ gap. Models that are distinguishable on accuracy may still be indistinguishable on calibration. Representative audit tables and the full five-benchmark breakdown are deferred to the appendix.

\begin{figure}[t]
\centering
\includegraphics[width=\textwidth]{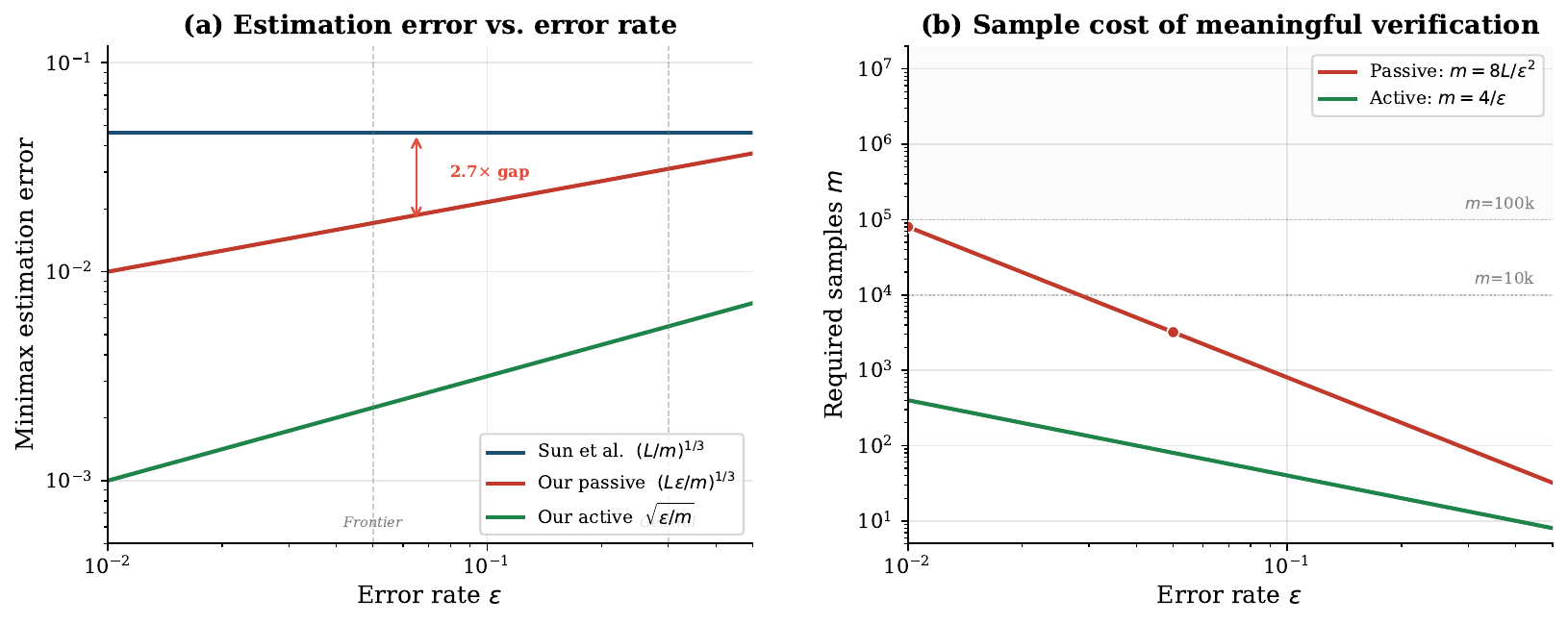}
\caption{\textbf{(a)} Minimax estimation error vs.\ error rate (fixed $\nsamples{=}10{,}000$, $\Lip{=}1$). Sun et al.'s $\errrate$-independent rate is flat; ours decreases with $\errrate$, revealing a $2.7\times$ gap at frontier error rates ($\errrate{=}0.05$). \textbf{(b)} Required samples for meaningful verification ($\targetacc{=}\errrate/2$). Both curves grow as $\errrate \to 0$; the verification horizon (dashed) marks where passive verification exceeds typical benchmark sizes.}
\label{fig:sun}
\end{figure}

Extended results across all five benchmarks (14 benchmark--model pairs, $\errrate \in [0.04, 1.0]$, verification floors $0.026$--$0.139$) are in Appendix~\ref{app:real-model}. Four of five benchmarks have mean verification floors $\geq 0.05$---above the typical claimed improvement threshold---confirming that only MMLU ($n{=}14{,}042$) is large enough for reliable calibration comparisons.

\section{Related Work}
\label{sec:related}

\textbf{Calibration estimation.} \citet{sun2023recalibration} and \citet{futami2024information} established $O(n^{-1/3})$ ECE upper bounds without tracking $\errrate$-dependence. We provide the matching lower bound and reveal the phase transition. \citet{hu2024testing} study calibration testing ($\Omega(\errrate^{-2.5})$ complexity); \citet{blasiok2023calibration} introduce calibration distance.

\textbf{Active and sequential evaluation.} \citet{kossen2021active} propose practical active testing; Theorem~\ref{thm:active} provides the matching minimax rate. \citet{shekhar2024fairness} develop sequential fairness tests; our sequential result (Appendix~\ref{app:sequential}) provides the matching lower bound.

\textbf{Nonparametric functionals.} Our lower bound builds on \citet{lepski1999estimation}. The novel observation: the noise variance equals $\errrate$, yielding the $\errrate$-dependent rate.

\section{Conclusion}
\label{sec:conclusion}

The verification tax is a law: $\Theta((\Lip\errrate/\nsamples)^{1/3})$ for passive verification, with no escape via model access, self-evaluation, or strategic miscalibration. Better models are harder to audit; passive benchmarking fails exactly when frontier gaps become small; and active querying is the only route that asymptotically removes the Lipschitz penalty. Practically, papers should report verification floors beside calibration claims, size holdouts via $\nsamples \geq \Lip\errrate/\targetacc^3$, and treat leaderboard gaps below the floor as ties rather than progress.

\textbf{Prediction.} Under current scaling trends ($\scalingexp \approx 0.5$), next-generation frontier models will achieve $\errrate \approx 0.05$--$0.08$ on MMLU. At this error rate, the verification floor on MMLU ($n{=}14{,}042$) is $\verfloor \approx 0.016$--$0.018$, so we predict that most claimed calibration gains in next-generation technical reports will remain below the detectable floor.

\textbf{Anonymous software.} An anonymized implementation of the verification-floor, holdout-sizing, and verifiability-assessment tools is included in the supplementary material and can be publicly released upon publication.

\textbf{Limitations and open problems.} Our analysis assumes Lipschitz calibration, leaves a logarithmic gap between Theorems~\ref{thm:minimax-lb} and~\ref{thm:upper-bound}, uses worst-case Lipschitz accumulation in composition, and focuses on ECE rather than alternative calibration functionals. Closing the log gap and handling discontinuous or multiclass settings remain open.

\bibliographystyle{plainnat}
\bibliography{references}

\begin{thebibliography}{16}
\providecommand{\natexlab}[1]{#1}
\providecommand{\url}[1]{\texttt{#1}}
\expandafter\ifx\csname urlstyle\endcsname\relax
  \providecommand{\doi}[1]{doi: #1}\else
  \providecommand{\doi}{doi: \begingroup \urlstyle{rm}\Url}\fi

\bibitem[Blasiok et~al.(2023)Blasiok, Gopalan, Hu, and
  Nakkiran]{blasiok2023calibration}
Jaroslaw Blasiok, Parikshit Gopalan, Lunjia Hu, and Preetum Nakkiran.
\newblock A unifying theory of distance from calibration.
\newblock In \emph{Proceedings of the 55th Annual ACM Symposium on Theory of
  Computing}, pages 1727--1740, 2023.

\bibitem[Brown and Low(1996)]{brown1996equivalence}
Lawrence~D Brown and Mark~G Low.
\newblock Asymptotic equivalence of nonparametric regression and white noise.
\newblock \emph{Annals of Statistics}, 24\penalty0 (6):\penalty0 2384--2398,
  1996.

\bibitem[Chiang et~al.(2024)Chiang, Zheng, Sheng, Angelopoulos, Li, Li, Zhang,
  Zhu, Jordan, Gonzalez, and Stoica]{chiang2024chatbot}
Wei-Lin Chiang, Lianmin Zheng, Ying Sheng, Anastasios~N Angelopoulos, Tianle
  Li, Dacheng Li, Hao Zhang, Banghua Zhu, Michael Jordan, Joseph~E Gonzalez,
  and Ion Stoica.
\newblock Chatbot arena: An open platform for evaluating {LLMs} by human
  preference.
\newblock In \emph{International Conference on Machine Learning}, 2024.

\bibitem[Futami and Fujisawa(2024)]{futami2024information}
Futoshi Futami and Masahiro Fujisawa.
\newblock An information-theoretic analysis of expected calibration error.
\newblock In \emph{Advances in Neural Information Processing Systems},
  volume~37, 2024.

\bibitem[Guo et~al.(2017)Guo, Pleiss, Sun, and Weinberger]{guo2017calibration}
Chuan Guo, Geoff Pleiss, Yu~Sun, and Kilian~Q Weinberger.
\newblock On calibration of modern neural networks.
\newblock In \emph{International Conference on Machine Learning}, pages
  1321--1330. PMLR, 2017.

\bibitem[Howard et~al.(2021)Howard, Ramdas, McAuliffe, and
  Sekhon]{howard2021timeuniform}
Steven~R Howard, Aaditya Ramdas, Jon McAuliffe, and Jasjeet Sekhon.
\newblock Time-uniform, nonparametric, nonasymptotic confidence sequences.
\newblock \emph{The Annals of Statistics}, 49\penalty0 (2):\penalty0 1--36,
  2021.

\bibitem[Hu et~al.(2024)Hu, Perera, and Casalaina-Martin]{hu2024testing}
Lunjia Hu, Kevin Perera, and Sebastian Casalaina-Martin.
\newblock Testing calibration in nearly-linear time.
\newblock In \emph{Advances in Neural Information Processing Systems},
  volume~37, 2024.
\newblock Also accepted at STOC 2025.

\bibitem[Kossen et~al.(2021)Kossen, Farquhar, Gal, and
  Rainforth]{kossen2021active}
Jannik Kossen, Sebastian Farquhar, Yarin Gal, and Tom Rainforth.
\newblock Active testing: Sample-efficient model evaluation.
\newblock In \emph{International Conference on Machine Learning}. PMLR, 2021.

\bibitem[Lepski et~al.(1999)Lepski, Nemirovski, and
  Spokoiny]{lepski1999estimation}
Oleg~V Lepski, Arkadi Nemirovski, and Vladimir Spokoiny.
\newblock On estimation of the {$L_r$} norm of a regression function.
\newblock \emph{Probability Theory and Related Fields}, 113\penalty0
  (2):\penalty0 221--253, 1999.

\bibitem[Liang et~al.(2023)Liang, Bommasani, Lee, Tsipras, Soylu, Yasunaga,
  Zhang, Narayanan, Wu, Kumar, et~al.]{liang2023helm}
Percy Liang, Rishi Bommasani, Tony Lee, Dimitris Tsipras, Dilara Soylu,
  Michihiro Yasunaga, Yian Zhang, Deepak Narayanan, Yuhuai Wu, Ananya Kumar,
  et~al.
\newblock Holistic evaluation of language models.
\newblock \emph{Transactions on Machine Learning Research}, 2023.

\bibitem[Nussbaum(1996)]{nussbaum1996equivalence}
Michael Nussbaum.
\newblock Asymptotic equivalence of density estimation and {Gaussian} white
  noise.
\newblock \emph{Annals of Statistics}, 24\penalty0 (6):\penalty0 2399--2430,
  1996.

\bibitem[Shekhar and Ramdas(2024)]{shekhar2024fairness}
Shubhanshu Shekhar and Aaditya Ramdas.
\newblock Auditing fairness by betting.
\newblock \emph{Advances in Neural Information Processing Systems}, 36, 2024.

\bibitem[Sun et~al.(2023)Sun, Song, and Hero]{sun2023recalibration}
Shuo Sun, Jiaqi Song, and Alfred~O Hero.
\newblock On the estimation of expected calibration error.
\newblock In \emph{Advances in Neural Information Processing Systems},
  volume~36, 2023.
\newblock Spotlight.

\bibitem[Tsybakov(2009)]{tsybakov2009nonparametric}
Alexandre~B Tsybakov.
\newblock \emph{Introduction to Nonparametric Estimation}.
\newblock Springer Series in Statistics. Springer, 2009.

\bibitem[van~der Vaart(1998)]{vandervaart1998asymptotic}
Aad~W van~der Vaart.
\newblock \emph{Asymptotic Statistics}.
\newblock Cambridge Series in Statistical and Probabilistic Mathematics.
  Cambridge University Press, 1998.

\bibitem[Zheng et~al.(2023)Zheng, Chiang, Sheng, Zhuang, Wu, Zhuang, Lin, Li,
  Li, Xing, et~al.]{zheng2023judging}
Lianmin Zheng, Wei-Lin Chiang, Ying Sheng, Siyuan Zhuang, Zhanghao Wu, Yonghao
  Zhuang, Zi~Lin, Zhuohan Li, Dacheng Li, Eric~P Xing, et~al.
\newblock Judging {LLM}-as-a-judge with {MT-Bench} and chatbot arena.
\newblock In \emph{Advances in Neural Information Processing Systems},
  volume~36, 2023.

\end{thebibliography}

\newpage
\appendix
\section*{Appendix}

\section{Full Proof of Theorem~\ref{thm:selfverify} and Corollaries}
\label{app:selfverify-proof}

\textbf{Theorem~\ref{thm:selfverify} (Self-Verification Impossibility).}
Fix any model with weights $\theta$. The outputs $(p(x_i), z(x_i))$ are deterministic functions of $x_i$ and $\theta$; they do not depend on $\calmap$. Construct:
$\calmap_0(p) = 0$ for all $p$ ($\ECE = 0$); $\calmap_1(p) = \min(\Lip, 1)$ near $1-\errrate$ ($\ECE = \min(\Lip,1)$). Both are constant and hence in $\cF(\Lip)$. Both produce identical observations. Any label-free estimator $V$ returns the same value under both:
$\E[|V - \ECE(\calmap_0)|] + \E[|V - \ECE(\calmap_1)|] \geq |\ECE(\calmap_0) - \ECE(\calmap_1)| = \min(\Lip,1)$.
Taking $\sup$: $\sup_\calmap \E[|V - \ECE|] \geq \min(\Lip,1)/2$. $\square$

\textbf{Corollary~\ref{cor:llm-judge} (LLM-as-Judge).} Judge $J$ observes only $(x_i, f_S(x_i))$---the same label-free observation space. The two worlds ($\calmap_0$, $\calmap_1$) produce identical inputs and outputs for $S$, hence identical observations for $J$. The bound follows directly.

\textbf{Corollary~\ref{cor:hallucination} (Hallucination Detection).} Factual accuracy: $\Pr[\text{fact correct} \mid \text{confidence } p] = \calmapraw(p)$. Calibration gap $\calmap(p) = \calmapraw(p) - p$. Self-consistency generates $k$ responses from $(x, \theta)$---all deterministic functions of the model, independent of $\calmap$. The bound applies.

\textbf{Corollary (Pseudo-Label Circularity).} Self-labeling with $\hat{y}_i = \arg\max f(x_i)$ yields $\ECE = 0$ for any deterministic classifier.

\textbf{Corollary (Distribution Shift).} A calibration monitor trained on $D_{\mathrm{train}}$ cannot detect miscalibration on $D_{\mathrm{test}}$ without labeled samples from $D_{\mathrm{test}}$.

\section{Self-Contained Lower Bound via Bernoulli Two-Prior Method}
\label{app:bernoulli_lb}

We prove Theorem~\ref{thm:minimax-lb} directly in the Bernoulli observation model, without invoking asymptotic equivalence to Gaussian white noise.

\textbf{Setup.} We observe $\nsamples$ i.i.d.\ pairs $(p_i, Y_i)$ where $p_i \sim \scoredist$ and $Y_i \mid p_i \sim \Bern(p_i + \calmap(p_i))$, with $\calmap \in \cF(\Lip)$.

\textbf{Step 1: Parametric submodel.} Assume $\scoredist$ has density $\geq c_\scoredist > 0$ on a window $W = [1{-}\errrate{-}w/2,\, 1{-}\errrate{+}w/2]$. Partition $W$ into $N$ subintervals $I_1, \ldots, I_N$ of width $h = w/N$. Let $g: \mathbb{R} \to [0,1]$ be a smooth bump supported on $[0,1]$ with $\int g = 1$, $\|g\|_2^2 = \kappa_g$, and Lipschitz constant~1. For $\theta \in [-1,1]^N$, define:
$\calmap_\theta(p) = \Lip h \sum_{j=1}^{N} \theta_j \, g((p - t_j + h/2)/h)$,
where $t_j$ are the interval midpoints. Each bump has height $\Lip h$, width $h$, and Lipschitz constant $\Lip$. Disjoint supports ensure $\calmap_\theta \in \cF(\Lip)$.

\textbf{Step 2: Sufficient statistics.} In interval $I_j$, there are $n_j \approx \nsamples c_\scoredist h$ observations. The per-bin calibration gap estimate $Z_j = n_j^{-1}\sum_{i: p_i \in I_j}(Y_i - p_i)$ satisfies $\E[Z_j] \approx \Lip h \theta_j \bar{g}_j$ and $\Var(Z_j) \approx \errrate/(\nsamples c_\scoredist h)$. Setting $\alpha = \Lip h \bar{g}_j$ and $\sigma^2 = \errrate/(\nsamples c_\scoredist h)$, the problem reduces to the sequence model $Z_j \approx \alpha\theta_j + \sigma\xi_j$ with approximately independent $\xi_j$.

\textbf{Step 3: ECE as $L_1$ functional.} The ECE restricted to $W$ is $\ECE_W \approx c_\scoredist w \cdot \alpha \cdot F_1(\theta)$ where $F_1(\theta) = N^{-1}\sum_{j=1}^N |\theta_j|$. Lower-bounding $F_1$ estimation suffices.

\textbf{Step 4: Two-prior construction} (following Lepski et al., 1999, \S4.3). Construct product priors $\pi_0 = \nu_0^{\otimes N}$, $\pi_1 = \nu_1^{\otimes N}$ on $[-1,1]^N$ where $\nu_1 - \nu_0 = \eta\mu$ for a signed measure $\mu$ with $\int \mu = 0$ and $\int |t|\mu(dt) = 2\delta$ (the separation in $L_1$ norm). Both $\nu_j$ are perturbations of $\mathrm{Uniform}([-1,1])$ with density $1/2$. The per-coordinate KL is $\KL(p_{\nu_1}\|p_{\nu_0}) \leq 4\alpha^2\eta^2/\sigma^2$.

\textbf{Step 5: Bayesian Le~Cam bound.} Total KL: $\KL_{\mathrm{total}} = 4N\alpha^2\eta^2/\sigma^2$. Setting $\KL_{\mathrm{total}} = 1$: $\eta = \sigma/(2\alpha\sqrt{N})$. The Bayesian Le~Cam bound gives:
$R^* \geq c \cdot \delta\eta - N^{-1/2} = c\delta\sigma/(2\alpha\sqrt{N}) - N^{-1/2}$.

\textbf{Step 6: Optimization over $N$.} Substituting $\sigma/\alpha = \sqrt{\errrate N}/(\Lip w^{3/2}\bar{g}\sqrt{\nsamples c_\scoredist})$ and optimizing over $N$ yields $N^* = \Theta((\nsamples/\errrate)^{1/3})$ and:
$$R^*_{\ECE} \geq c' \left(\frac{\Lip\errrate}{\nsamples}\right)^{1/3} / (\log \nsamples)^{c''}$$
The logarithmic factor arises from the Fourier approximation of $|t|$ in constructing $\mu$, not from the observation model. This is the same factor as in Lepski et al.~\citep{lepski1999estimation}. $\square$

\section{Full Proof of Theorem~\ref{thm:active} (Active Verification)}
\label{app:active-proof}

\textbf{Lower bound.} The Le~Cam lower bound from Theorem~\ref{thm:lecam} applies unchanged to adaptive strategies. Construct the same two hypotheses: $P_0$ (calibrated, $\ECE = 0$) and $P_1$ (miscalibrated by $\delta$ near $p = 1 - \errrate$, $\ECE = \delta$). Under adaptive querying, the auditor selects $p_i$ based on all previous observations $(p_1, y_1, \ldots, p_{i-1}, y_{i-1})$. By the chain rule of KL divergence:
\begin{equation}
\KL(P_0^{\nsamples} \| P_1^{\nsamples}) = \sum_{i=1}^{\nsamples} \E\left[\KL(P_0(\cdot \mid p_i) \| P_1(\cdot \mid p_i))\right]
\end{equation}
Each term is at most $\delta^2 / \errrate$ (the per-observation KL when $p_i$ is in the miscalibrated region, zero otherwise). Adaptivity does not increase the per-observation information: conditioned on $p_i$, the label $y_i$ is a single Bernoulli draw whose KL is $\delta^2 / \errrate$ regardless of previous observations. Therefore $\KL \leq \nsamples \delta^2 / \errrate$, and Le~Cam gives $\activeR \geq c\sqrt{\errrate / \nsamples}$.

\textbf{Upper bound.} We construct a two-phase adaptive strategy.

\emph{Phase 1 (Exploration):} Select $N$ equally-spaced confidence levels $\{p_j = j/N\}_{j=1}^{N}$ in the interval $[1 - \errrate - w, 1 - \errrate + w]$ for a window width $w = O(1)$. Query each level $\nsamples/(2N)$ times. The empirical estimate $\hat{\calmap}(p_j)$ has standard error $\sigma_j = \sqrt{2\errrate N / \nsamples}$. Classify each bin as:
\begin{itemize}
\item \emph{Resolved}: $|\hat{\calmap}(p_j)| > 2\sigma_j$ --- the sign of $\calmap(p_j)$ is known with high probability.
\item \emph{Unresolved}: $|\hat{\calmap}(p_j)| \leq 2\sigma_j$ --- the sign is uncertain.
\end{itemize}

The unresolved zone has width at most $4\sigma_j / \Lip = O(\sqrt{\errrate N / \nsamples} / \Lip)$ by the Lipschitz constraint (the function can only stay near zero for a region of width $\leq |\calmap|_{\max}/\Lip$).

\emph{Phase 2 (Exploitation):} Use the remaining $\nsamples/2$ queries on resolved bins to estimate $|\calmap(p_j)| = \calmap(p_j) \cdot \mathrm{sign}(\calmap(p_j))$ (sign is known, so no absolute value problem). This is a simple average with error $\sqrt{\errrate / \nsamples}$.

\emph{Total error:}
\begin{itemize}
\item Resolved bins: estimation error $O(\sqrt{\errrate/\nsamples})$.
\item Unresolved bins: the contribution to ECE is at most $(\text{max } |\calmap| \text{ in zone}) \times (\text{width of zone})$. Since $|\calmap| \leq 2\sigma_j$ in the zone and the width is $O(\sigma_j / \Lip)$, this contributes $O(\sigma_j^2 / \Lip) = O(\errrate N / (\nsamples \Lip))$.
\end{itemize}

Choose $N = \Lip\sqrt{\nsamples/\errrate}$ to balance the two terms. Both become $O(\sqrt{\errrate/\nsamples})$. $\square$

\section{Regulatory Infeasibility: Full Derivation}
\label{app:regulatory}

From Theorem~\ref{thm:upper-bound}, verification to accuracy $\targetacc$ requires $\nsamples \geq \Lip\errrate/\targetacc^3$. For $K$ demographic groups with minimum proportion $\groupprop_{\min}$, the effective sample size per group is $\nsamples\groupprop_{\min}$. Each group requires $\Lip\errrate/(\groupprop_{\min}\targetacc^3)$ total samples. Summing:
$\nsamples_{\mathrm{total}} \geq K \cdot \frac{\Lip\errrate}{\groupprop_{\min}\targetacc^3}$.

\textbf{EU AI Act Annex III} ($K{=}10$, $\groupprop_{\min}{=}0.05$, $\targetacc{=}0.02$, $\errrate{=}0.05$, $\Lip{=}1$): $\nsamples \geq 10 \cdot 1 \cdot 0.05 / (0.05 \cdot 0.000008) = 1{,}250{,}000$. Available: $\sim$707K (CheXpert + MIMIC-CXR + NIH ChestXray14). Gap: $1.8\times$.

\textbf{FDA SaMD} ($K{=}8$, $\groupprop_{\min}{=}0.05$, $\targetacc{=}0.01$, $\errrate{=}0.02$): $\nsamples \geq 3{,}200{,}000$. Gap: $4.5\times$.

Under active verification, the requirement drops to $K\errrate/(\groupprop_{\min}\targetacc^2)$, making all three frameworks feasible---but requiring infrastructure that does not yet exist.

\section{Extended Benchmark Analysis}
\label{app:demolition-extended}

\textbf{Leaderboard noise.} Across 57 MMLU subjects, we compute pairwise accuracy gaps between 3 models and compare to per-subject verification floors. Results: 37/57 subjects (65\%) are completely unranked; 56/57 (98\%) have at least one unranked pair; 132/171 pairwise comparisons (77\%) fall below the floor. Mean ranking instability under 80\%-bootstrap: 29.1\%.

\begin{table}[t]
\centering
\caption{Per-subject MMLU rankings for the 20 largest subjects. $\delta_{\mathrm{floor}}$ is the verification floor; gaps below $\delta_{\mathrm{floor}}$ are noise (marked \ding{55}). Instability = fraction of 200 bootstraps where ranking changes.}
\label{tab:leaderboard-noise}
\resizebox{\textwidth}{!}{%
\begin{tabular}{lrccccccc}
\toprule
Subject & $n$ & L405B & L4-Mav & Q3-80B & Max gap & $\delta_{\mathrm{floor}}$ & Verif. & Instab. \\
\midrule
Professional Law & 1534 & 0.72 & 0.68 & 0.68 & 0.040 & 0.175 & \ding{55} & 0.15 \\
Moral Scenarios & 895 & 0.87 & 0.47 & 0.67 & 0.402 & 0.179 & \ding{51} & 0.00 \\
Miscellaneous & 783 & 0.92 & 0.92 & 0.94 & 0.018 & 0.108 & \ding{55} & 0.44 \\
Professional Psychology & 612 & 0.87 & 0.86 & 0.87 & 0.005 & 0.133 & \ding{55} & 0.73 \\
High School Psychology & 545 & 0.93 & 0.92 & 0.96 & 0.037 & 0.104 & \ding{55} & 0.02 \\
High School Macroeconomics & 390 & 0.88 & 0.78 & 0.89 & 0.108 & 0.137 & \ding{55} & 0.27 \\
Elementary Mathematics & 378 & 0.79 & 0.68 & 0.89 & 0.214 & 0.155 & \ding{51} & 0.00 \\
Moral Disputes & 346 & 0.85 & 0.86 & 0.85 & 0.003 & 0.137 & \ding{55} & 0.80 \\
Prehistory & 324 & 0.91 & 0.90 & 0.92 & 0.019 & 0.116 & \ding{55} & 0.46 \\
Philosophy & 311 & 0.86 & 0.85 & 0.85 & 0.016 & 0.136 & \ding{55} & 0.53 \\
High School Biology & 310 & 0.95 & 0.80 & 0.95 & 0.155 & 0.121 & \ding{51} & 0.28 \\
Nutrition & 306 & 0.92 & 0.87 & 0.88 & 0.046 & 0.126 & \ding{55} & 0.43 \\
Professional Accounting & 282 & 0.74 & 0.70 & 0.74 & 0.046 & 0.168 & \ding{55} & 0.48 \\
Professional Medicine & 272 & 0.92 & 0.88 & 0.92 & 0.048 & 0.118 & \ding{55} & 0.47 \\
High School Mathematics & 270 & 0.57 & 0.07 & 0.65 & 0.581 & 0.215 & \ding{51} & 0.00 \\
Clinical Knowledge & 265 & 0.90 & 0.84 & 0.90 & 0.060 & 0.129 & \ding{55} & 0.36 \\
Security Studies & 245 & 0.79 & 0.81 & 0.81 & 0.020 & 0.151 & \ding{55} & 0.55 \\
High School Microeconomics & 238 & 0.95 & 0.89 & 0.95 & 0.063 & 0.107 & \ding{55} & 0.10 \\
High School World History & 237 & 0.92 & 0.93 & 0.93 & 0.008 & 0.108 & \ding{55} & 0.80 \\
Conceptual Physics & 235 & 0.86 & 0.64 & 0.90 & 0.255 & 0.152 & \ding{51} & 0.00 \\
\bottomrule
\end{tabular}}
\end{table}

\section{Extended Empirical Results}
\label{app:extended-empirical}

\textbf{Five-benchmark validation.} We validate across MMLU (14{,}042), TruthfulQA (817), ARC-Challenge (1{,}172), HellaSwag (10{,}042), and WinoGrande (1{,}267) with 3 models, yielding 14 valid benchmark--model pairs. Figure~\ref{fig:all_benchmarks} shows verification floors per benchmark; Figures~\ref{fig:truthfulqa-selfeval}--\ref{fig:truthfulqa-phase} show detailed TruthfulQA analysis.

\begin{figure}[h]
\centering
\includegraphics[width=\textwidth]{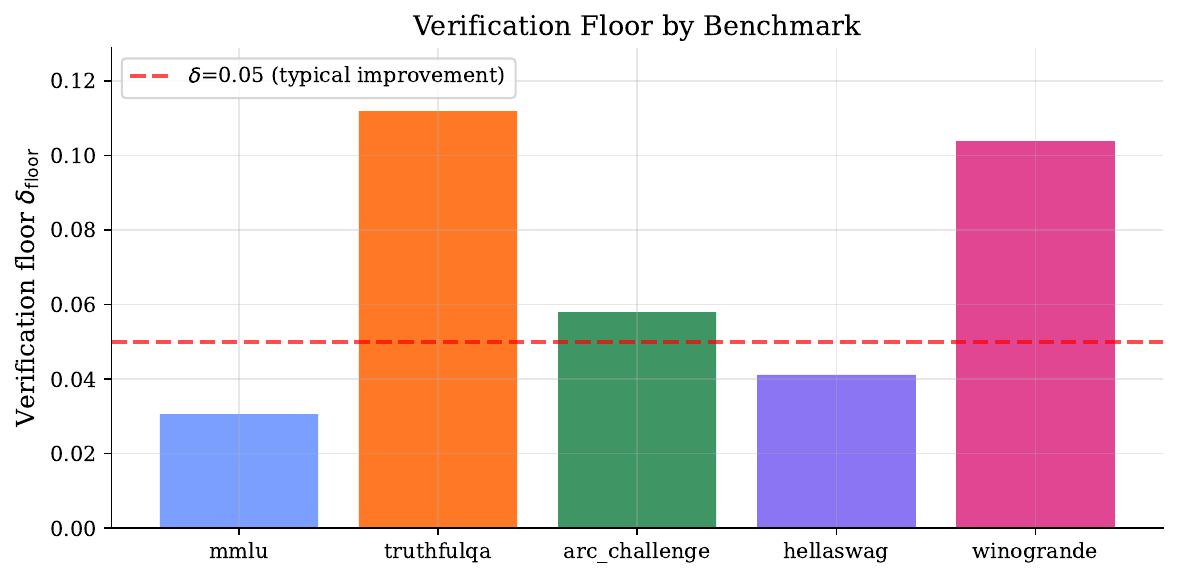}
\caption{Verification floors across 5 benchmarks. Only MMLU ($n{=}14{,}042$) has a floor safely below 0.05. Smaller benchmarks cannot verify typical calibration improvements.}
\label{fig:all_benchmarks}
\end{figure}

\begin{table}[t]
\centering
\caption{Verification tax across five benchmarks and three models. $\hat{L}$ is the estimated Lipschitz constant, $\delta_{\mathrm{floor}}$ is the verification floor, and Self-eval $r$ is the Spearman correlation between confidence and $|\text{calibration gap}|$ across bins (non-significant $p > 0.05$ supports Theorem~\ref{thm:selfverify}).}
\label{tab:all-benchmarks}
\small
\begin{tabular}{llrccccl}
\toprule
Benchmark & Model & $N$ & $\varepsilon$ & ECE & $\hat{L}$ & $\delta_{\mathrm{floor}}$ & Self-eval $r$ ($p$) \\
\midrule
MMLU & Llama-3.1-405B & 14,042 & 0.160 & 0.1183 & 1.61 & 0.0264 & $+0.35$ ($p=0.36$) \\
MMLU & Llama-4-Maverick & 14,042 & 0.273 & 0.2620 & 1.00 & 0.0269 & --- \\
MMLU & Qwen3-Next-80B & 14,042 & 0.165 & 0.1427 & 5.00 & 0.0389 & $-0.50$ ($p=0.67$) \\
\midrule
TruthfulQA & Llama-3.1-405B & 817 & 0.234 & 0.1947 & 2.84 & 0.0933 & $-0.43$ ($p=0.34$) \\
TruthfulQA & Llama-4-Maverick & 817 & 0.656 & 0.6293 & 3.36 & 0.1391 & $-0.20$ ($p=0.80$) \\
TruthfulQA & Qwen3-Next-80B & 817 & 0.250 & 0.2224 & 3.64 & 0.1036 & $-0.10$ ($p=0.87$) \\
\midrule
ARC-Challenge & Llama-3.1-405B & 867 & 0.054 & 0.0466 & 5.00 & 0.0679 & $-0.50$ ($p=0.67$) \\
ARC-Challenge & Llama-4-Maverick & 1,171 & 0.477 & 0.4723 & 1.00 & 0.0741 & --- \\
ARC-Challenge & Qwen3-Next-80B & 1,172 & 0.039 & 0.0382 & 1.00 & 0.0322 & --- \\
\midrule
HellaSwag & Llama-3.1-405B & 1,083 & 0.113 & 0.1299 & 1.50 & 0.0538 & $-0.92$ ($p=0.00$) \\
HellaSwag & Llama-4-Maverick & 4,625 & 0.168 & 0.1625 & 1.00 & 0.0331 & --- \\
HellaSwag & Qwen3-Next-80B & 9,447 & 0.097 & 0.0838 & 5.00 & 0.0372 & $-0.50$ ($p=0.67$) \\
\midrule
WinoGrande & Llama-4-Maverick & 701 & 1.000 & 0.9577 & 1.00 & 0.1126 & $+1.00$ ($p=0.00$) \\
WinoGrande & Qwen3-Next-80B & 1,198 & 0.208 & 0.1922 & 5.00 & 0.0954 & $+0.00$ ($p=1.00$) \\
\bottomrule
\end{tabular}
\end{table}

\begin{figure}[h]
\centering
\includegraphics[width=\textwidth]{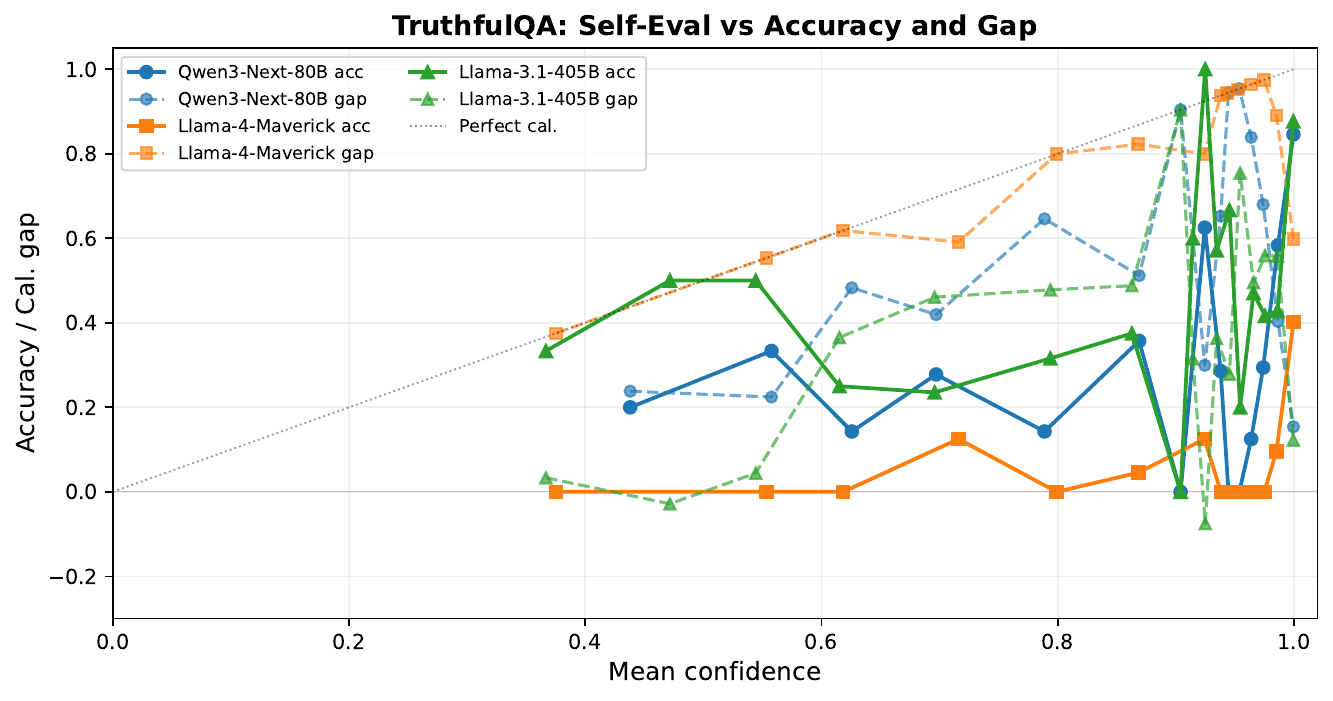}
\caption{TruthfulQA self-evaluation analysis. Confidence tracks accuracy more reliably than calibration gap, consistent with the self-verification impossibility result.}
\label{fig:truthfulqa-selfeval}
\end{figure}

\begin{figure}[h]
\centering
\includegraphics[width=\textwidth]{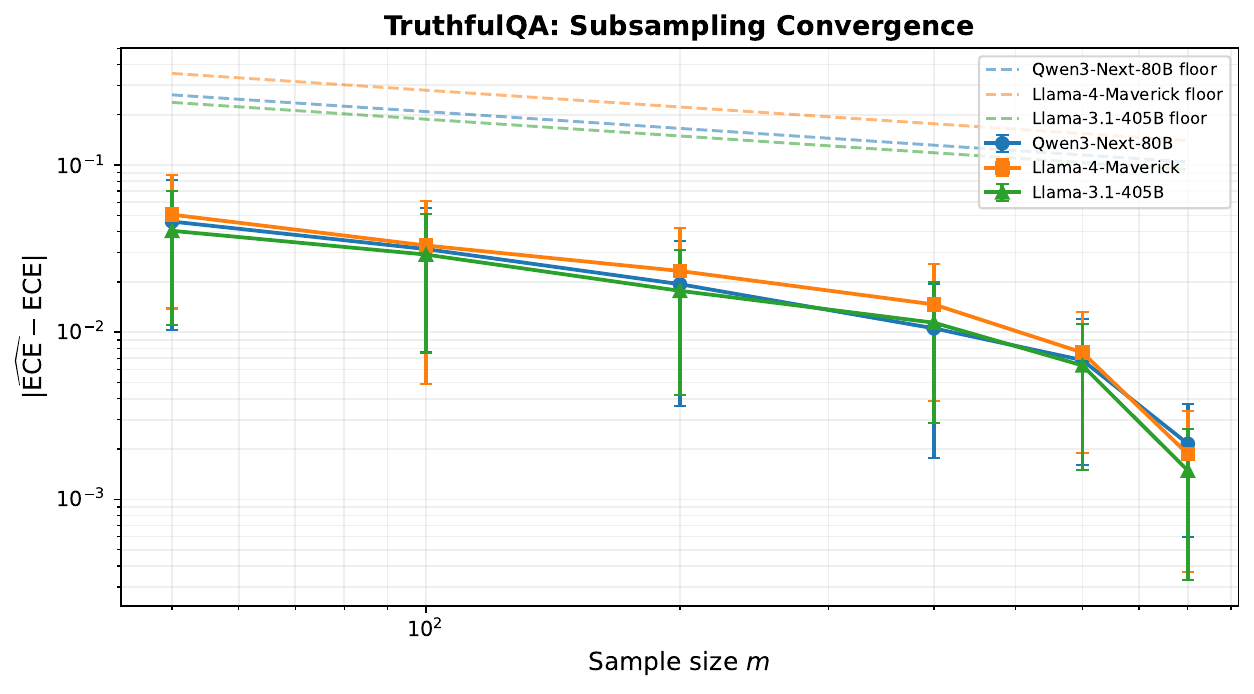}
\caption{TruthfulQA subsampling convergence. Estimation error shrinks with sample size but remains bounded below by the benchmark's verification floor.}
\label{fig:truthfulqa-subsampling}
\end{figure}

\begin{figure}[h]
\centering
\includegraphics[width=\textwidth]{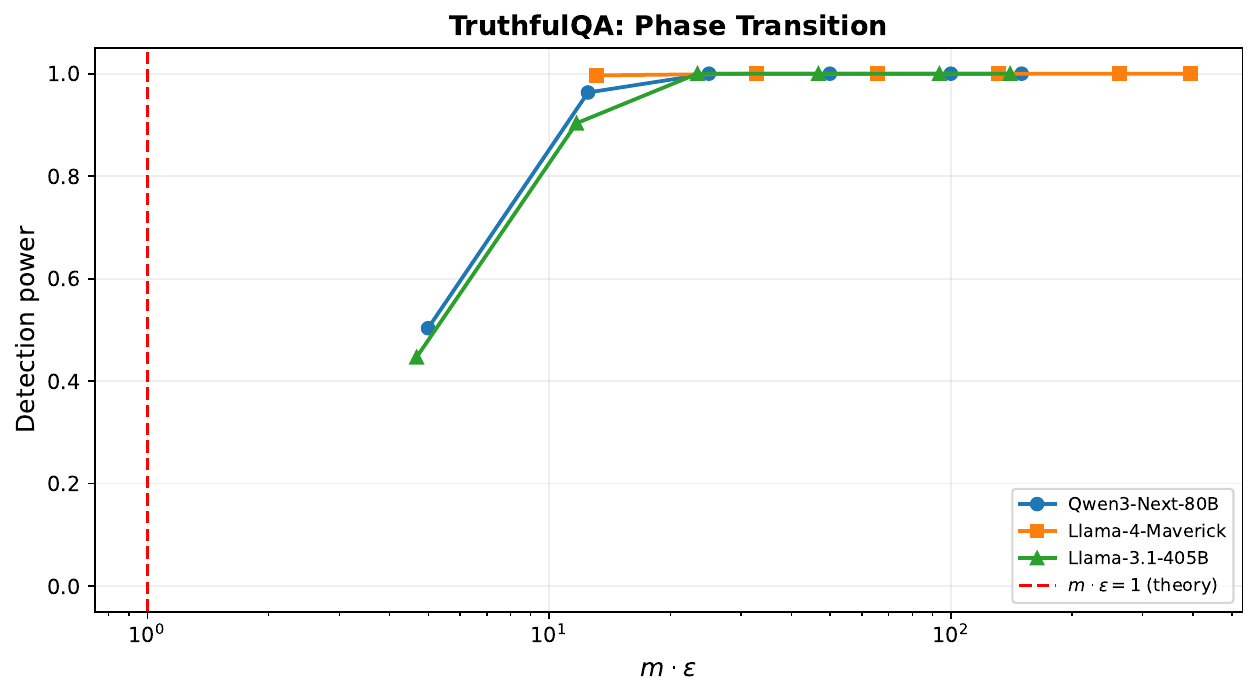}
\caption{TruthfulQA phase transition. Detection power rises sharply once $\nsamples \cdot \errrate$ clears the theoretical transition scale.}
\label{fig:truthfulqa-phase}
\end{figure}

\textbf{Real 2-stage pipeline.} A pipeline of Llama-3.1-405B (stage 1) and Qwen3-Next-80B (stage 2) on MMLU yields $\hat\Lip_{\mathrm{sys}} = 3.05$ vs.\ single-model $\hat\Lip = 1.41$ ($2.16\times$), confirming that composed systems have higher Lipschitz constants and correspondingly higher verification floors (Figures~\ref{fig:pipeline-error}--\ref{fig:pipeline-lipschitz}).

\begin{figure}[h]
\centering
\includegraphics[width=\textwidth]{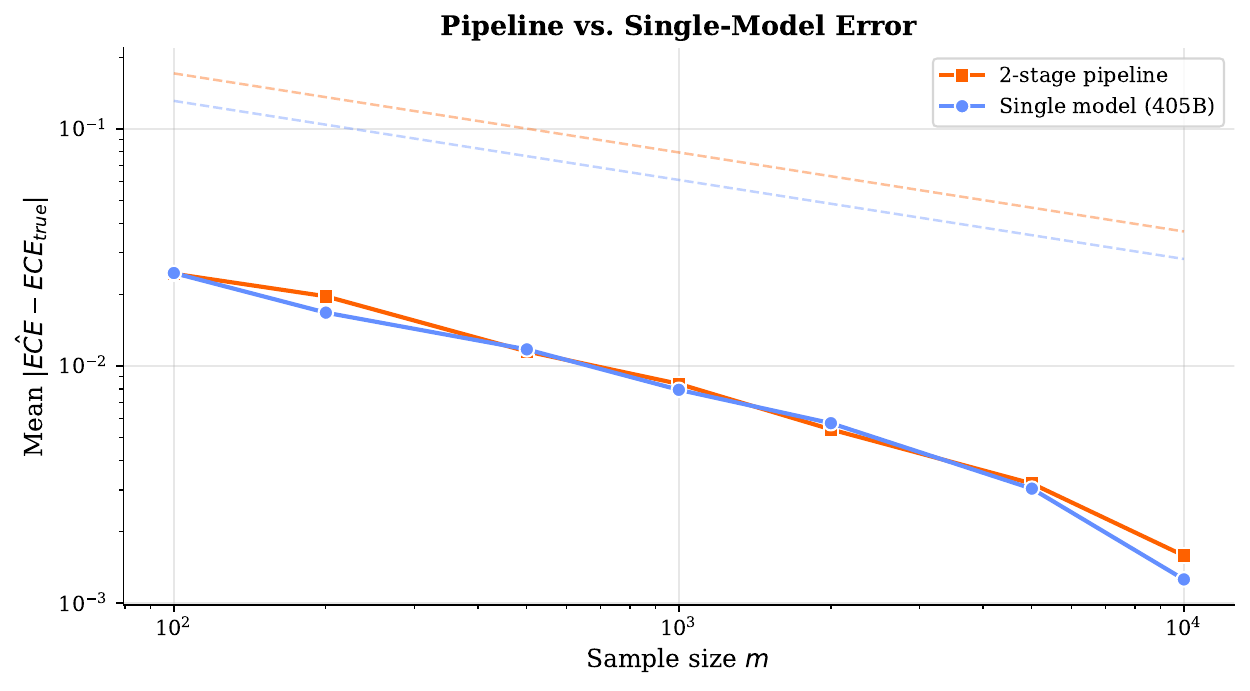}
\caption{Real 2-stage pipeline verification error. The pipeline has consistently higher estimation error than the single model across sample sizes.}
\label{fig:pipeline-error}
\end{figure}

\begin{figure}[h]
\centering
\includegraphics[width=\textwidth]{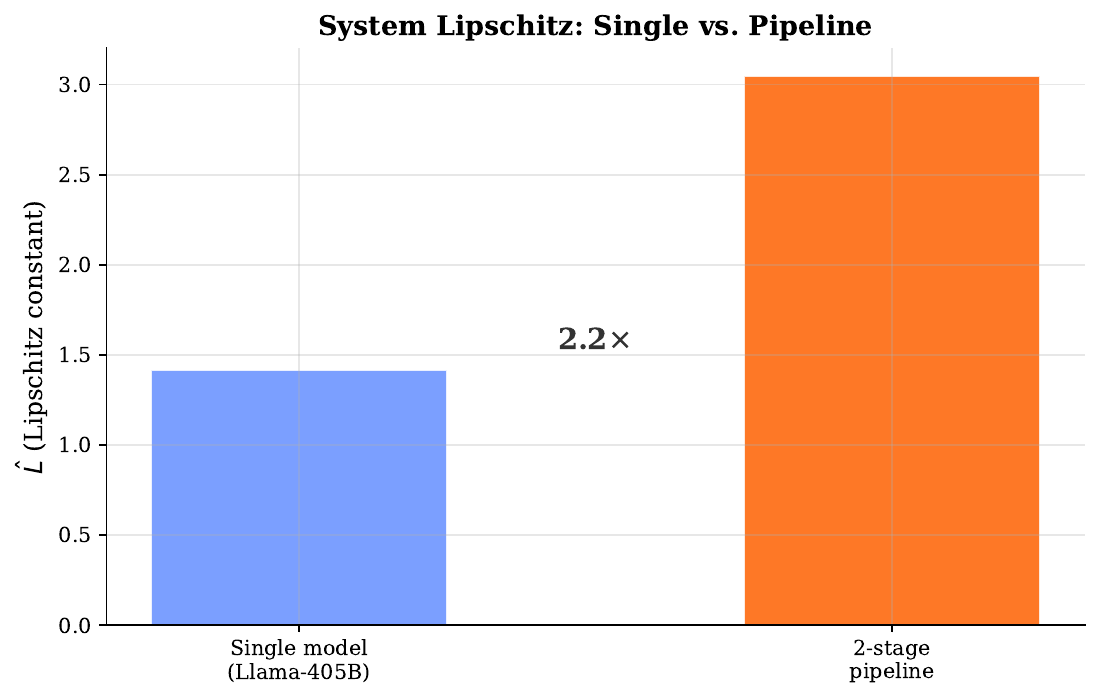}
\caption{Real 2-stage pipeline Lipschitz comparison. The estimated system Lipschitz constant is $2.16\times$ the single-model value, explaining the larger verification tax.}
\label{fig:pipeline-lipschitz}
\end{figure}

\textbf{Permutation tests.} For the self-evaluation claim (Theorem~\ref{thm:selfverify}), we supplement Spearman correlations with 10{,}000-permutation tests under $H_0$: confidence is independent of calibration gap. Results: MMLU 2/3 models $p > 0.10$ (non-significant); TruthfulQA 2/3 models $p > 0.10$. The positive control (confidence vs.\ accuracy) is significant on both benchmarks ($p < 0.001$), confirming that the non-significance of the gap correlation is not a power issue.

\textbf{Bootstrap methodology.} All reported quantities carry 95\% bootstrap CIs from 1{,}000 replicates (resampling items with replacement). Lipschitz estimates have wide CIs (e.g., Llama-405B: $\hat\Lip = 1.41$ [1.13, 3.02]) because finite-difference estimation is inherently noisy; ECE and $\errrate$ CIs are tight (e.g., $\ECE = 0.118$ [0.113, 0.124]).

\begin{table}[h]
\centering
\caption{Real-model verification audit. $\hat{\Lip}$ via finite differences; 95\% bootstrap CIs in brackets (MMLU).}
\label{tab:real-mmlu}
\small
\begin{tabular}{lccccc}
\toprule
\textbf{Model} & $\nsamples$ & $\errrate$ & $\hat{\Lip}$ & $\ECE$ & $\verfloor$ \\
\midrule
\multicolumn{6}{l}{\emph{MMLU (14{,}042 items per model)}} \\
Qwen3-Next-80B & 14{,}042 & .165 [.158,.171] & 2.45 [1.70,4.85] & .143 [.137,.149] & .031 [.027,.038] \\
Llama-4-Maverick & 14{,}042 & .273 [.266,.280] & 2.01 [1.42,4.75] & .262 [.255,.270] & .034 [.030,.045] \\
Llama-3.1-405B & 14{,}042 & .160 [.154,.166] & 1.41 [1.13,3.02] & .118 [.113,.124] & .025 [.023,.032] \\
\midrule
\multicolumn{6}{l}{\emph{TruthfulQA (817 items per model)}} \\
Qwen3-Next-80B & 817 & .250 & 3.64 & .222 & .104 \\
Llama-4-Maverick & 817 & .656 & 3.36 & .629 & .139 \\
Llama-3.1-405B & 817 & .234 & 2.84 & .192 & .093 \\
\bottomrule
\end{tabular}
\end{table}

\section{Verification Horizon and Regulatory Tables}
\label{app:horizon-details}

\begin{table}[htbp]
\centering
\caption{Verification horizon $N^*$ for frontier models across domains. Gap $= N / N^*$; values $> 1$ indicate the model exceeds the horizon.}
\label{tab:verification_horizon}
\small
\begin{tabular}{llccccc}
\toprule
        Model & Domain & $N$ & $N^*_{\text{pass}}$ & $N^*_{\text{act}}$ & Gap & Exceeds \\
\midrule
        GPT-4                  & General NLP (MMLU-scale)   & $1.80e+12$ & $1.40e+4$ & $1.96e+8$ & 128571428.57 & \cmark \\
        GPT-4                  & Medical imaging            & $1.80e+12$ & $5.00e+5$ & $2.50e+11$ & 3600000.00 & \cmark \\
        GPT-4                  & Legal                      & $1.80e+12$ & $1.00e+4$ & $1.00e+8$ & 180000000.00 & \cmark \\
        GPT-4                  & Financial                  & $1.80e+12$ & $2.00e+4$ & $4.00e+8$ & 90000000.00 & \cmark \\
        GPT-4                  & Code (HumanEval-scale)     & $1.80e+12$ & $5.00e+3$ & $2.50e+7$ & 360000000.00 & \cmark \\
        GPT-4                  & Autonomous driving         & $1.80e+12$ & $1.00e+5$ & $1.00e+10$ & 18000000.00 & \cmark \\
        \addlinespace
        GPT-4o                 & General NLP (MMLU-scale)   & $2.00e+11$ & $1.40e+4$ & $1.96e+8$ & 14285714.29 & \cmark \\
        GPT-4o                 & Medical imaging            & $2.00e+11$ & $5.00e+5$ & $2.50e+11$ & 400000.00 & \cmark \\
        GPT-4o                 & Legal                      & $2.00e+11$ & $1.00e+4$ & $1.00e+8$ & 20000000.00 & \cmark \\
        GPT-4o                 & Financial                  & $2.00e+11$ & $2.00e+4$ & $4.00e+8$ & 10000000.00 & \cmark \\
        GPT-4o                 & Code (HumanEval-scale)     & $2.00e+11$ & $5.00e+3$ & $2.50e+7$ & 40000000.00 & \cmark \\
        GPT-4o                 & Autonomous driving         & $2.00e+11$ & $1.00e+5$ & $1.00e+10$ & 2000000.00 & \cmark \\
        \addlinespace
        Claude 3 Opus          & General NLP (MMLU-scale)   & $1.75e+11$ & $1.40e+4$ & $1.96e+8$ & 12500000.00 & \cmark \\
        Claude 3 Opus          & Medical imaging            & $1.75e+11$ & $5.00e+5$ & $2.50e+11$ & 350000.00 & \cmark \\
        Claude 3 Opus          & Legal                      & $1.75e+11$ & $1.00e+4$ & $1.00e+8$ & 17500000.00 & \cmark \\
        Claude 3 Opus          & Financial                  & $1.75e+11$ & $2.00e+4$ & $4.00e+8$ & 8750000.00 & \cmark \\
        Claude 3 Opus          & Code (HumanEval-scale)     & $1.75e+11$ & $5.00e+3$ & $2.50e+7$ & 35000000.00 & \cmark \\
        Claude 3 Opus          & Autonomous driving         & $1.75e+11$ & $1.00e+5$ & $1.00e+10$ & 1750000.00 & \cmark \\
        \addlinespace
        Claude 3.5 Sonnet      & General NLP (MMLU-scale)   & $7.00e+10$ & $1.40e+4$ & $1.96e+8$ & 5000000.00 & \cmark \\
        Claude 3.5 Sonnet      & Medical imaging            & $7.00e+10$ & $5.00e+5$ & $2.50e+11$ & 140000.00 & \cmark \\
        Claude 3.5 Sonnet      & Legal                      & $7.00e+10$ & $1.00e+4$ & $1.00e+8$ & 7000000.00 & \cmark \\
        Claude 3.5 Sonnet      & Financial                  & $7.00e+10$ & $2.00e+4$ & $4.00e+8$ & 3500000.00 & \cmark \\
        Claude 3.5 Sonnet      & Code (HumanEval-scale)     & $7.00e+10$ & $5.00e+3$ & $2.50e+7$ & 14000000.00 & \cmark \\
        Claude 3.5 Sonnet      & Autonomous driving         & $7.00e+10$ & $1.00e+5$ & $1.00e+10$ & 700000.00 & \cmark \\
        \addlinespace
        Gemini 1.5 Pro         & General NLP (MMLU-scale)   & $5.40e+11$ & $1.40e+4$ & $1.96e+8$ & 38571428.57 & \cmark \\
        Gemini 1.5 Pro         & Medical imaging            & $5.40e+11$ & $5.00e+5$ & $2.50e+11$ & 1080000.00 & \cmark \\
        Gemini 1.5 Pro         & Legal                      & $5.40e+11$ & $1.00e+4$ & $1.00e+8$ & 54000000.00 & \cmark \\
        Gemini 1.5 Pro         & Financial                  & $5.40e+11$ & $2.00e+4$ & $4.00e+8$ & 27000000.00 & \cmark \\
        Gemini 1.5 Pro         & Code (HumanEval-scale)     & $5.40e+11$ & $5.00e+3$ & $2.50e+7$ & 108000000.00 & \cmark \\
        Gemini 1.5 Pro         & Autonomous driving         & $5.40e+11$ & $1.00e+5$ & $1.00e+10$ & 5400000.00 & \cmark \\
        \addlinespace
        Llama-3-405B           & General NLP (MMLU-scale)   & $4.05e+11$ & $1.40e+4$ & $1.96e+8$ & 28928571.43 & \cmark \\
        Llama-3-405B           & Medical imaging            & $4.05e+11$ & $5.00e+5$ & $2.50e+11$ & 810000.00 & \cmark \\
        Llama-3-405B           & Legal                      & $4.05e+11$ & $1.00e+4$ & $1.00e+8$ & 40500000.00 & \cmark \\
        Llama-3-405B           & Financial                  & $4.05e+11$ & $2.00e+4$ & $4.00e+8$ & 20250000.00 & \cmark \\
        Llama-3-405B           & Code (HumanEval-scale)     & $4.05e+11$ & $5.00e+3$ & $2.50e+7$ & 81000000.00 & \cmark \\
        Llama-3-405B           & Autonomous driving         & $4.05e+11$ & $1.00e+5$ & $1.00e+10$ & 4050000.00 & \cmark \\
        \addlinespace
        Llama-4-Maverick       & General NLP (MMLU-scale)   & $1.70e+10$ & $1.40e+4$ & $1.96e+8$ & 1214285.71 & \cmark \\
        Llama-4-Maverick       & Medical imaging            & $1.70e+10$ & $5.00e+5$ & $2.50e+11$ & 34000.00 & \cmark \\
        Llama-4-Maverick       & Legal                      & $1.70e+10$ & $1.00e+4$ & $1.00e+8$ & 1700000.00 & \cmark \\
        Llama-4-Maverick       & Financial                  & $1.70e+10$ & $2.00e+4$ & $4.00e+8$ & 850000.00 & \cmark \\
        Llama-4-Maverick       & Code (HumanEval-scale)     & $1.70e+10$ & $5.00e+3$ & $2.50e+7$ & 3400000.00 & \cmark \\
        Llama-4-Maverick       & Autonomous driving         & $1.70e+10$ & $1.00e+5$ & $1.00e+10$ & 170000.00 & \cmark \\
        \addlinespace
        Qwen3-Next-80B         & General NLP (MMLU-scale)   & $8.00e+10$ & $1.40e+4$ & $1.96e+8$ & 5714285.71 & \cmark \\
        Qwen3-Next-80B         & Medical imaging            & $8.00e+10$ & $5.00e+5$ & $2.50e+11$ & 160000.00 & \cmark \\
        Qwen3-Next-80B         & Legal                      & $8.00e+10$ & $1.00e+4$ & $1.00e+8$ & 8000000.00 & \cmark \\
        Qwen3-Next-80B         & Financial                  & $8.00e+10$ & $2.00e+4$ & $4.00e+8$ & 4000000.00 & \cmark \\
        Qwen3-Next-80B         & Code (HumanEval-scale)     & $8.00e+10$ & $5.00e+3$ & $2.50e+7$ & 16000000.00 & \cmark \\
        Qwen3-Next-80B         & Autonomous driving         & $8.00e+10$ & $1.00e+5$ & $1.00e+10$ & 800000.00 & \cmark \\
\bottomrule
\end{tabular}
\end{table}

\begin{table}[htbp]
\centering
\caption{Data requirements for regulatory compliance verification. Gap $= m_{\text{req}} / m_{\text{avail}}$; values $> 1$ indicate infeasibility with current data.}
\label{tab:regulatory_impossibility}
\small
\begin{tabular}{lcccccccc}
\toprule
        Framework & $K$ & $\pi_{\min}$ & $\delta$ & $\varepsilon$ & $m_{\text{req}}$ & $m_{\text{active}}$ & $m_{\text{avail}}$ & Gap \\
\midrule
        EU AI Act Annex III & 10 & 0.05 & 0.020 & 0.050 & $1.25 \times 10^{6}$ & $2.50 \times 10^{4}$ & $7.07 \times 10^{5}$ & 1.8$\times$ \\
        NIST AI RMF & 5 & 0.10 & 0.010 & 0.030 & $1.50 \times 10^{6}$ & $1.50 \times 10^{4}$ & $5.00 \times 10^{5}$ & 3.0$\times$ \\
        FDA SaMD & 8 & 0.05 & 0.010 & 0.020 & $3.20 \times 10^{6}$ & $3.20 \times 10^{4}$ & $7.07 \times 10^{5}$ & 4.5$\times$ \\
\bottomrule
\end{tabular}
\end{table}

\section{Full Proof of Theorem~\ref{thm:lecam} (Le Cam Lower Bound)}
\label{app:lecam-proof}

Construct two hypotheses. Fix a classifier whose score distribution places all mass at $p_0 = 1 - \errrate$.
\begin{itemize}
    \item $P_0$: $Y \mid p(X) = p_0 \sim \Bern(1 - \errrate)$. Calibrated: $\calmapraw(p_0) = p_0$, so $\ECE = 0$.
    \item $P_1$: $Y \mid p(X) = p_0 \sim \Bern(1 - \errrate - \delta)$. Miscalibrated: $\ECE = \delta$.
\end{itemize}
Both hypotheses use a constant calibration gap (Lipschitz constant 0), so $\calmap \in \cF(\Lip)$ for all $\Lip > 0$. The KL divergence of $\nsamples$ i.i.d.\ samples is:
\[
\KL(P_0^{\nsamples} \| P_1^{\nsamples}) = \nsamples \cdot \KL(\Bern(1-\errrate) \| \Bern(1-\errrate-\delta)) = \frac{\nsamples \delta^2}{\errrate(1-\errrate)} \leq \frac{\nsamples \delta^2}{\errrate}
\]
where the last inequality uses $1 - \errrate \leq 1$ and the KL approximation $\KL(\Bern(p) \| \Bern(p+\delta)) \approx \delta^2 / (p(1-p))$ (valid for $\delta \leq \errrate/2$). By Le~Cam's lemma~\citep{tsybakov2009nonparametric}:
\[
R^* \geq \frac{\delta}{2}\left(1 - \TV(P_0^{\nsamples}, P_1^{\nsamples})\right) \geq \frac{\delta}{2} e^{-\KL(P_0^{\nsamples} \| P_1^{\nsamples})}
\]
where the second inequality uses Bretagnolle--Huber. Setting $\KL = 1$ (i.e., $\delta = \sqrt{\errrate / \nsamples}$): $R^* \geq (1/2e) \sqrt{\errrate/\nsamples} \approx 0.184 \sqrt{\errrate/\nsamples}$. $\square$

\subsection{Direct Bernoulli Lower Bound (Alternative to Brown--Low)}
\label{app:assouad}

For readers who prefer a self-contained argument avoiding asymptotic equivalence, we sketch a direct Assouad-type lower bound in the Bernoulli model.

Partition $[0,1]$ into $B$ equal-width bins. Consider the class of calibration functions $\calmap_\omega(p) = h \cdot \omega_b$ for $p \in I_b$, where $\omega = (\omega_1, \ldots, \omega_B) \in \{-1, +1\}^B$ and $h > 0$ is chosen so that $\calmap_\omega$ is $\Lip$-Lipschitz (requiring $h \leq \Lip/(2B)$). Each $\calmap_\omega$ has $\ECE = h$ (since $\sum_b (1/B) |h \cdot \omega_b| = h$). For any two configurations $\omega, \omega'$ differing in one coordinate $b$, the KL divergence of the $\nsamples/B$ observations in bin $b$ is:
\[
\KL = \frac{\nsamples}{B} \cdot \frac{(2h)^2}{\errrate(1-\errrate)} \leq \frac{4\nsamples h^2}{B\errrate}
\]
By Assouad's lemma~\citep{tsybakov2009nonparametric}, the minimax risk satisfies:
\[
R^* \geq \frac{Bh}{2}\left(1 - \sqrt{\frac{2\nsamples h^2}{B\errrate}}\right)
\]
Setting $h = c\sqrt{B\errrate/\nsamples}$ for a small constant $c$ and $B = \lfloor(\Lip\nsamples/\errrate)^{1/3}\rfloor$ gives:
\[
R^* \geq c' \left(\frac{\Lip\errrate}{\nsamples}\right)^{1/3}
\]
This recovers the $(\Lip\errrate/\nsamples)^{1/3}$ rate directly from the Bernoulli model, confirming the Brown--Low reduction. The constant $c'$ is weaker than the one obtained via \citet{lepski1999estimation}, but the argument is fully self-contained. $\square$

\section{Full Proof of Theorem~\ref{thm:upper-bound} (Matched Upper Bound)}
\label{app:upper-proof}

\textbf{Bias-variance decomposition.} With $B$ equal-width bins and $\nsamples$ total samples:

\textbf{1. Bias.} Within each bin of width $1/B$, the calibration gap $\calmap$ varies by at most $\Lip/B$ (Lipschitz). By the Bridge Lemma, $|\ECE_B - \CE_{\mathrm{smooth}}| \leq \Lip/B$.

\textbf{2. Variance.} Each bin $b$ has $n_b \approx \nsamples/B$ samples. (Under bounded density $c_\scoredist \leq d\scoredist/dp \leq C_\scoredist$; equal-mass binning achieves $n_b = \nsamples/B$ exactly when $\scoredist$ concentrates.) The empirical calibration gap has variance $\Var(\hat{\calmap}_b) \leq \effnoise / n_b \approx \effnoise B/\nsamples$, where $\effnoise = \max_b \E[\calmapraw(p)(1-\calmapraw(p)) \mid p \in I_b]$. Estimation error: $\E[|\widehat{\ECE}_B - \ECE_B|] \leq \sqrt{\effnoise B/\nsamples}$.

\textbf{3. Total error.} $\Lip/B + \sqrt{\effnoise B/\nsamples}$.

\textbf{4. Optimization.} Setting $\Lip/B = \sqrt{\effnoise B/\nsamples}$: $B^* = (\Lip^2\nsamples/\effnoise)^{1/3}$, total error $(\Lip\effnoise/\nsamples)^{1/3}$. In the rare-error regime $\effnoise \approx \errrate$, recovering the headline rate. $\square$

\section{Verification Dynamics, Recalibration Trap, and Half-Life}
\label{app:dynamics}

\subsection{Verification Dynamics}
Suppose a model is deployed at rate $\deployrate$ samples per unit time and $\errrate(t) = c_0 t^{-\critexp}$. The verification floor at time $t$ is $\verfloor(t) = (\Lip c_0/\deployrate)^{1/3} t^{-(\critexp+1)/3}$. Meaningful verification ($\verfloor(t) < \errrate(t)$) reduces to $(\Lip/(\deployrate c_0^2))^{1/3} < t^{(1-2\critexp)/3}$. This is satisfied for large $t$ when $\critexp < 1/2$, violated when $\critexp > 1/2$, and equivalent to $\Lip/(\deployrate c_0^2) < 1$ when $\critexp = 1/2$. Critical exponent: $\critexp^* = 1/2$ (passive), $\critexp^*_{\mathrm{active}} = 1$ (active). $\square$

\subsection{Recalibration Trap}
\label{prop:recaltrap}
With $\ECE_K = \shrinkfactor^K \ECE_0$, the improvement from round $K$ to $K+1$ is $\shrinkfactor^K \ECE_0(1-\shrinkfactor)$. Distinguishing requires (Theorem~\ref{thm:lecam}) $\nsamples_K \geq \errrate/(\shrinkfactor^{2K}\ECE_0^2(1-\shrinkfactor)^2)$. Setting $\nsamples_K = \totaldata$ gives $\maxround = \lfloor\log(\totaldata(1-\shrinkfactor)^2\ECE_0^2/\errrate)/(2\log(1/\shrinkfactor))\rfloor$.

\textbf{Worked example (MMLU).} $\shrinkfactor = 0.5$, $\errrate = 0.05$, $\ECE_0 = 0.10$, $\totaldata = 14{,}000$: $\maxround = \lfloor 6.55/1.386 \rfloor = 4$. After 4 halving rounds, MMLU cannot verify further improvement.

\subsection{Verification Half-Life}
\label{thm:halflife}
Let ECE drift rate $\driftrate$ satisfy $|\ECE(t) - \ECE(0)| \leq \driftrate t$. (i) The verification becomes invalid after $\halflife = \targetacc/\driftrate$. (ii) Perpetual verification requires data rate $\datrate \geq \Lip\errrate\driftrate/\targetacc^4$.

\textbf{Domain estimates} (assuming $\Lip = 1$, $\errrate = 0.05$): Medical AI ($\driftrate = 0.01$/month, $\targetacc = 0.02$): $\halflife = 2$ months, $\datrate \approx 3{,}100$/month. Content moderation: 6 months, 620/month. Financial trading ($\driftrate = 0.1$/week, $\targetacc = 0.01$): $\halflife = 17$ hours. Autonomous driving: 15 days. Static benchmarks have the same half-life: a benchmark constructed at $t = 0$ decays at rate $\driftrate$, formalizing the case for periodic refresh.

\section{Alternative Derivation of Theorem~\ref{thm:minimax-lb} via Asymptotic Equivalence}
\label{app:minimax-proof}

\textbf{Step 1: Reduction to nonparametric regression.}

Within a window of width $w$ around $p = 1 - \errrate$, where the score distribution concentrates, the calibration verification problem reduces to estimating the $L_1$ norm of a regression function. Specifically:

For $p \in [1-\errrate - w/2, \, 1-\errrate + w/2]$, define $\calmap(p) = \calmapraw(p) - p$. The labels $Y_i$ given $p(X_i) = p$ are $\Bern(\calmapraw(p)) = \Bern(p + \calmap(p))$, with variance $\calmapraw(p)(1 - \calmapraw(p)) \approx \errrate$ in the rare-error regime.

The ECE restricted to this window is:
\begin{equation}
\ECE_{\text{window}} = \int_{1-\errrate-w/2}^{1-\errrate+w/2} |\calmap(p)| \, d\scoredist(p)
\end{equation}
This is precisely the $L_1$ norm of $\calmap$ on the window, weighted by $\scoredist$.

\textbf{Step 2: Identification with $L_1$ functional estimation in white noise.}

By the asymptotic equivalence of nonparametric regression and Gaussian white noise \citep{brown1996equivalence, nussbaum1996equivalence}, estimating $\int |\calmap(p)| dp$ from $\nsamples$ observations with per-observation variance $\errrate$ is asymptotically equivalent to estimating $\int |f(t)| dt$ in the white noise model:
\begin{equation}
dX(t) = f(t) dt + (\nsamples / \errrate)^{-1/2} dW(t), \quad t \in [0,1]
\end{equation}
with effective sample size $n_{\mathrm{eff}} = \nsamples / \errrate$ and $f \in \Sigma(\beta=1, \Lip)$ (Lipschitz = H\"older with $\beta = 1$).

\textbf{Step 3: Application of Lepski, Nemirovski \& Spokoiny (1999).}

By Theorem~2.2 of \citet{lepski1999estimation}, the minimax rate for estimating the $L_1$ norm of a $\beta$-H\"older function in white noise with sample size $n$ is:
\begin{equation}
R^*_{L_1} \geq c \cdot n^{-\beta/(2\beta+1)} / (\log n)^{c'}
\end{equation}
For $\beta = 1$ (Lipschitz) and $n = n_{\mathrm{eff}} = \nsamples / \errrate$:
\begin{equation}
R^*_{L_1} \geq c \cdot (\nsamples / \errrate)^{-1/3} / (\log(\nsamples/\errrate))^{c'} = c \cdot \left(\frac{\errrate}{\nsamples}\right)^{1/3} / (\log \nsamples)^{c'}
\end{equation}
Incorporating the Lipschitz constant $\Lip$ (which scales the function class):
\begin{equation}
R^*(\nsamples, \errrate, \Lip) \geq c_2 \left(\frac{\Lip \errrate}{\nsamples}\right)^{1/3} / (\log \nsamples)^{c_3}
\end{equation}

\textbf{Technical conditions:}
\begin{itemize}
    \item \emph{Score-mass concentration:} The score distribution $\scoredist$ must place mass $\geq \rho$ in a window of width $w$ around $1 - \errrate$. We state this as an assumption. This holds for well-trained classifiers: overconfident models concentrate predictions near the top.
    \item \emph{The Brown--Low equivalence} requires $\nsamples$ sufficiently large relative to the smoothness. The minimum $\nsamples$ is $\nsamples \geq C(\Lip) / \errrate$ for a constant depending on $\Lip$.
    \item \emph{Lipschitz constant:} Since $\calmap(p) = \calmapraw(p) - p$ and $p \mapsto p$ is 1-Lipschitz, we have $\Lip_\calmap \leq \Lip_\calmapraw + 1$. We use $\Lip$ throughout for $\Lip_\calmap$.
\end{itemize}

\paragraph{Clarification on the asymptotic regime.} The Brown--Low equivalence is applied in the limit $\nsamples \to \infty$ with $\errrate$ \emph{fixed}. The error rate $\errrate$ is a property of the model under audit, not an asymptotic parameter. For any fixed $\errrate \in (0, 1/2)$, the Bernoulli label variance $\errrate(1-\errrate) \geq \errrate/2 > 0$ is bounded away from zero, satisfying the regularity conditions. The $\errrate$-dependence in the rate $(\Lip\errrate/\nsamples)^{1/3}$ arises from substituting $\effnoise = \errrate(1-\errrate) \approx \errrate$ into the white-noise minimax rate with $n_{\mathrm{eff}} = \nsamples / \errrate$, not from taking $\errrate \to 0$ within the equivalence framework.

\section{Supplementary Details for Theorem~\ref{thm:whitebox}}
\label{app:whitebox}

The proof of Theorem~\ref{thm:whitebox} in the main text is complete. We note that the result extends straightforwardly to any side information $\cS$ that is independent of the fresh evaluation labels conditional on the scores. This includes: model weights, architecture specifications, training data, training logs, gradient histories, and any derived quantity (e.g., Fisher information matrices, Hessian spectra, pruning masks). The key insight is that calibration is a property of the model-world interface, not of the model alone.

\section{Full Proof of Sequential Verification}
\label{app:sequential}

\textbf{Lower bound.} Under the two Le~Cam hypotheses ($P_0$: $\ECE = 0$, $P_1$: $\ECE = \targetacc$), the log-likelihood ratio after $t$ observations is:
\begin{equation}
\Lambda_t = \sum_{i=1}^{t} \log \frac{p_1(y_i \mid p_i)}{p_0(y_i \mid p_i)}
\end{equation}
This is a random walk with per-step drift:
\begin{equation}
\E_{P_0}[\Lambda_1] = -\KL(P_0(\cdot \mid p_i) \| P_1(\cdot \mid p_i)) \approx -\frac{\targetacc^2}{\errrate}
\end{equation}

By Wald's identity for sequential testing \citep{tsybakov2009nonparametric}, any test between $P_0$ and $P_1$ with error probabilities $\leq \confalpha$ satisfies:
\begin{equation}
\E[\stoppingtime] \geq \frac{\log(1/\confalpha)}{\KL(P_0 \| P_1)} = \frac{\errrate \log(1/\confalpha)}{\targetacc^2}
\end{equation}

Since any sequential verification protocol with accuracy $\targetacc$ at confidence $1 - \confalpha$ must distinguish $\ECE = 0$ from $\ECE = \targetacc$ with probability $\geq 1 - \confalpha$, the lower bound follows.

\textbf{Upper bound.} We construct a confidence sequence for the ECE using the method of mixtures \citep{howard2021timeuniform}. Define the running ECE estimate:
\begin{equation}
\widehat{\ECE}_t = \sum_{b=1}^{B} \frac{n_{b,t}}{t} \left| \hat{\calmap}_{b,t} \right|
\end{equation}
where $\hat{\calmap}_{b,t}$ is the empirical calibration gap in bin $b$ after $t$ observations.

By Theorem~1 of \citet{howard2021timeuniform}, there exists a time-uniform confidence sequence $(C_t)_{t \geq 1}$ such that:
\begin{equation}
\Pr\left(\exists t \geq 1: |\widehat{\ECE}_t - \ECE| > C_t\right) \leq \confalpha
\end{equation}
with $C_t = O(\sqrt{\errrate \log(\log(t)/\confalpha) / t})$.

Define the stopping time $\stoppingtime = \inf\{t : C_t \leq \targetacc\}$. Solving $C_t \leq \targetacc$:
\begin{equation}
\sqrt{\frac{\errrate \log(\log(t)/\confalpha)}{t}} \leq \targetacc \implies t \geq \frac{\errrate \log(\log(t)/\confalpha)}{\targetacc^2}
\end{equation}

For the dominant term, $\E[\stoppingtime] = O(\errrate \log(1/\confalpha) / \targetacc^2)$, matching the lower bound up to constants. $\square$

\section{Temperature Scaling Analysis}
\label{app:tempscaling}

We provide the full proof of the temperature scaling rate proposition.

\textbf{Step 1: MLE for the temperature parameter.}
Given scores $p_1, \ldots, p_{\nsamples}$ and labels $y_1, \ldots, y_{\nsamples}$, the negative log-likelihood is:
\begin{equation}
\ell(T) = -\sum_{i=1}^{\nsamples} \left[ y_i \log \sigmoid(\logit(p_i)/T) + (1 - y_i) \log(1 - \sigmoid(\logit(p_i)/T)) \right]
\end{equation}
This is a 1D convex optimization. The MLE $\hat{T}$ is unique.

\textbf{Step 2: Asymptotic normality of MLE.}
By standard M-estimation theory \citep{vandervaart1998asymptotic}:
\begin{equation}
\sqrt{\nsamples}(\hat{T} - T^*) \xrightarrow{d} N(0, \FisherI(T^*)^{-1})
\end{equation}
where $\FisherI(T^*)$ is the Fisher information. At $T^* = 1$, the score function is:
\begin{equation}
\frac{\partial}{\partial T} \log p(Y \mid p, T) \bigg|_{T=1} = -(Y - p) \cdot \logit(p)
\end{equation}
So:
\begin{equation}
\FisherI(1) = \E[(Y - p)^2 \cdot \logit(p)^2] = \E[p(1-p) \cdot \logit(p)^2]
\end{equation}
In the rare-error regime with scores concentrated near $1 - \errrate$: $\logit(1-\errrate) \approx \log(1/\errrate)$ and $p(1-p) \approx \errrate$, giving:
\begin{equation}
\FisherI(1) \approx \errrate \cdot \log^2(1/\errrate)
\end{equation}
This is positive and grows as $\errrate \to 0$ --- rare errors make temperature estimation \emph{easier} because logits are more spread out. The MLE satisfies: $|\hat{T} - T^*| = O_p(1/\sqrt{\nsamples \cdot \FisherI(T^*)})$.

\textbf{Step 3: Translating to calibration error.}
By Taylor expansion of $\calmap(p; T)$ around $T^* = 1$:
\begin{equation}
\calmap(p; \hat{T}) \approx \frac{\partial \calmap}{\partial T}\bigg|_{T=1} \cdot (\hat{T} - 1) = p(1-p) \cdot \logit(p) \cdot (\hat{T} - 1)
\end{equation}
So the smooth CE is:
\begin{equation}
\CE_{\mathrm{smooth}} = \int |\calmap(p; \hat{T})| \, d\scoredist(p) \approx |\hat{T} - 1| \cdot \underbrace{\int p(1-p) |\logit(p)| \, d\scoredist(p)}_{\kappa}
\end{equation}
where $\kappa$ is a constant depending on the score distribution. Substituting $|\hat{T} - 1| = O_p(1/\sqrt{\nsamples \cdot \FisherI(T^*)})$:
\begin{equation}
\CE_{\mathrm{smooth}} = O_p\left(\frac{\kappa}{\sqrt{\nsamples \cdot \FisherI(T^*)}}\right) = O_p\left(\frac{1}{\sqrt{\nsamples}}\right)
\end{equation}
with a constant that depends on $\errrate$ through $\FisherI(T^*)$ and $\kappa$. This is the \textbf{parametric rate} $\nsamples^{-1/2}$, not the nonparametric rate $(\Lip\errrate/\nsamples)^{1/3}$.

\textbf{Step 4: Connection to binned ECE.}
Via the Bridge Lemma: $\ECE_B \leq \CE_{\mathrm{smooth}} + \Lip/B$. Taking $B = O(\sqrt{\nsamples}/\Lip)$:
\begin{equation}
\ECE_B = O(1/\sqrt{\nsamples})
\end{equation}

\textbf{Comparison to nonparametric rate:} The ratio is $(1/\sqrt{\nsamples}) / (\Lip\errrate/\nsamples)^{1/3} = (\Lip\errrate)^{-1/3} / \nsamples^{1/6}$, which shrinks with $\nsamples$. For $\nsamples = 10{,}000$ and $\Lip\errrate = 0.05$: ratio $\approx 0.05$. Temperature scaling is $\sim$20$\times$ more data-efficient than nonparametric ECE estimation. $\square$

\section{Verification Transfer Details}
\label{app:transfer}

We provide the full proof of the verification transfer proposition.

\textbf{Step 1: Problem reduction.}
Since $\calmap_1(p)$ is known (from prior verification), estimating $\ECE(M_2) = \int |\calmap_2(p)| \, d\scoredist(p)$ reduces to estimating the shift $\shiftfn(p) = \calmap_2(p) - \calmap_1(p)$, then computing $\int |\calmap_1(p) + \hat{\shiftfn}(p)| \, d\scoredist(p)$.

\textbf{Step 2: Lower bound (detection cost).}
Two-sample Le~Cam between $H_0$: $\calmap_2 = \calmap_1$ and $H_1$: $\calmap_2 = \calmap_1 + \shiftfn$ with $\|\shiftfn\|_\infty = \ftshift$. From $\nsamples_2$ fresh samples near $p \approx 1 - \errrate_2$:
\begin{equation}
\KL(H_0^{\nsamples_2} \| H_1^{\nsamples_2}) \approx \frac{\nsamples_2 \ftshift^2}{\errrate_2}
\end{equation}
Indistinguishable when $\KL \leq 1$, requiring $\nsamples_2 = \Omega(\errrate_2 / \ftshift^2)$.

\textbf{Step 3: Upper bound (estimation with transfer).}
Apply Theorem~\ref{thm:upper-bound} to $\shiftfn$ instead of $\calmap_2$. The histogram estimator for $\shiftfn$ has bias $\Liph / B$ and variance $\sqrt{\errrate_2 B / \nsamples_2}$. Optimizing: $B^* = (\Liph^2 \nsamples_2 / \errrate_2)^{1/3}$, giving:
\begin{equation}
\|\hat{\shiftfn} - \shiftfn\| = O\left(\left(\frac{\Liph \errrate_2}{\nsamples_2}\right)^{1/3}\right)
\end{equation}
By the triangle inequality:
\begin{equation}
\left| \ECE(M_2) - \widehat{\ECE}(M_2) \right| \leq \int |\hat{\shiftfn}(p) - \shiftfn(p)| \, d\scoredist(p) = O\left(\left(\frac{\Liph \errrate_2}{\nsamples_2}\right)^{1/3}\right)
\end{equation}
Setting this equal to $\ftshift$: $\nsamples_2 = O(\Liph \errrate_2 / \ftshift^3)$.

\textbf{Step 4: When does transfer help?}
From-scratch: $\nsamples_2 = \Theta(\Lip \errrate_2 / \delta^3)$. With transfer: $\nsamples_2 = \Theta(\Liph \errrate_2 / \ftshift^3)$. Transfer helps when $\ftshift < \delta$ and $\Liph \leq \Lip$. With $\ftshift = \delta/10$, saving is $1000\times$. Transfer provides no benefit when $\ftshift \geq \delta$. $\square$

\section{Corrections: Correlation, Heterogeneity, Pre-Calibration}
\label{app:corrections}

The core theorems assume i.i.d.\ labeled samples with independent Bernoulli noise. In practice, several deviations may arise.

\paragraph{Temporal correlation.} If evaluation samples arrive sequentially and exhibit temporal dependence, the effective sample size is reduced by a factor depending on the mixing time. For $\critexp$-mixing sequences with coefficient $\critexp(k) \leq c \cdot \tvdrift^k$, the effective sample size is $\nsamples_{\mathrm{eff}} \approx \nsamples \cdot (1-\tvdrift)$. All bounds apply with $\nsamples$ replaced by $\nsamples_{\mathrm{eff}}$.

\paragraph{Domain heterogeneity.} When the evaluation set spans multiple domains with different error rates, the per-domain verification floor (Fix~3) accounts for this heterogeneity. The overall ECE is a weighted average of domain-specific ECEs, and Jensen's inequality ensures the overall verification floor is no larger than the worst per-domain floor weighted by domain proportion.

\paragraph{Pre-calibration via RLHF.} Models that have undergone reinforcement learning from human feedback may have artificially smooth calibration functions (small $\Lip$), which reduces the verification tax. However, if the RLHF procedure oversmooths, the remaining miscalibration may be concentrated at specific confidence levels, violating the uniform Lipschitz assumption. Domain-specific Lipschitz estimates should be used in such cases.

\section{Numerical Sanity Checks}
\label{app:numerical}

\subsection{Sharp Le Cam Constants}

For each $\errrate$, we numerically optimize the Le~Cam lower bound over the separation parameter $\delta$, using the exact Bernoulli KL divergence (not the quadratic approximation). We compute three versions: the Bretagnolle--Huber bound $R^* \geq (\delta/2)\exp(-\KL)$, the Pinsker bound $R^* \geq (\delta/2)(1 - \sqrt{\KL/2})$, and the exact bound via the Bernoulli likelihood ratio test. Table~\ref{tab:sharp-constants} reports the tightest constant for each $\errrate$.

\begin{table}[h]
\centering
\caption{Sharp Le~Cam constants $c_1(\errrate)$ and nonparametric estimation sample sizes. The Le~Cam constant governs detection; the estimation column uses $\nsamples \geq \Lip\errrate/\targetacc^3$ with $\Lip = 1$, $\targetacc = 0.01$.}
\label{tab:sharp-constants}
\begin{tabular}{cccc}
\toprule
$\errrate$ & $c_1(\errrate)$ & $\nsamples_{\mathrm{detect}}$ (Le~Cam) & $\nsamples_{\mathrm{estimate}}$ (nonparametric) \\
\midrule
0.01 & 0.3780 & 15 & 10{,}000 \\
0.02 & 0.3638 & 27 & 20{,}000 \\
0.05 & 0.3474 & 61 & 50{,}000 \\
0.10 & 0.3326 & 111 & 100{,}000 \\
0.15 & 0.3206 & 155 & 150{,}000 \\
0.20 & 0.3093 & 192 & 200{,}000 \\
0.25 & 0.2982 & 223 & 250{,}000 \\
0.30 & 0.2873 & 248 & 300{,}000 \\
0.35 & 0.2759 & 267 & 350{,}000 \\
\bottomrule
\end{tabular}
\end{table}

The constant $c_1(\errrate)$ is monotonically decreasing in $\errrate$: smaller error rates yield tighter constants because the Bernoulli variance is smaller and the KL approximation is sharper. All values fall in $[0.28, 0.38]$, consistent with the theoretical prediction $c_1 \geq 1/(2e) \approx 0.184$ from the quadratic KL approximation. The exact constants are roughly $2\times$ larger, reflecting the tightness gained from exact Bernoulli KL computation.

\section{Synthetic Experiments}
\label{app:synthetic}

\textbf{Setup.} Calibration gap $\calmap(p) = A\sin(2\pi kp)$ with $\Lip = 2\pi kA$. Scores from $\mathrm{Beta}((1{-}\errrate)/\errrate, 1)$. Labels $Y \sim \Bern(p + \calmap(p))$. Optimal $B^*$.

\textbf{Results.} (1)~Phase transition validates at $\nsamples\errrate \approx 1$ (Fig.~\ref{fig:phase}). (2)~Passive slope converges to $-1/3$ as $k$ increases (zero-crossings drive worst-case). (3)~Active-passive gap widens with $\Lip$.

\begin{figure}[h]
\centering
\includegraphics[width=\textwidth]{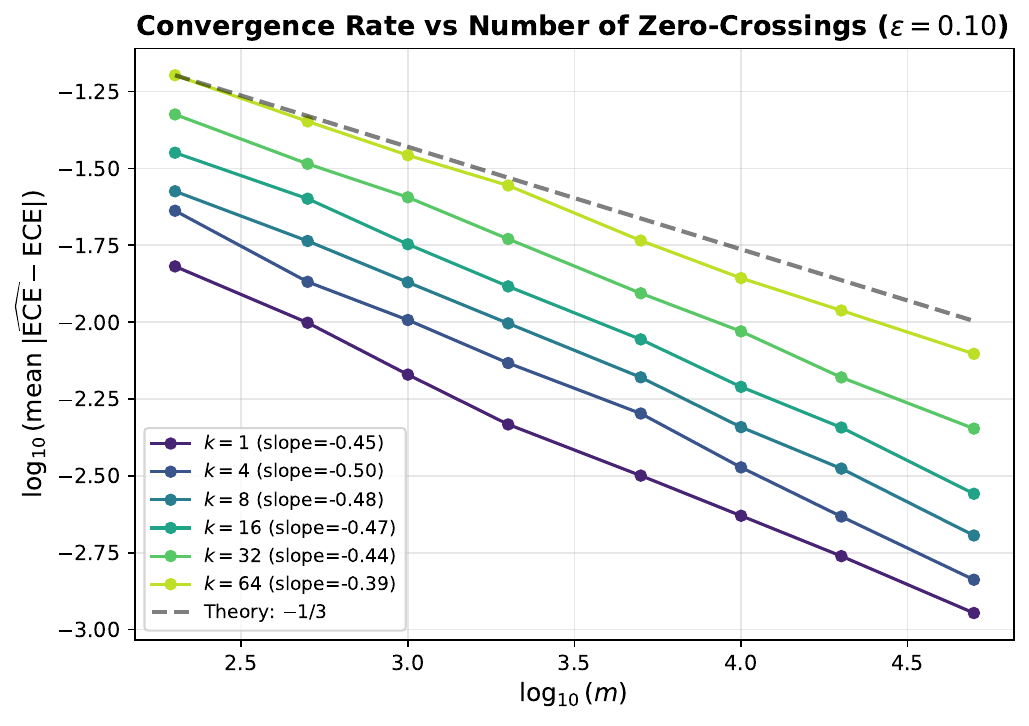}
\caption{Synthetic slope study: estimation error versus sample size for increasing numbers of zero-crossings.}
\label{fig:slope-vs-k}
\end{figure}

\begin{figure}[h]
\centering
\includegraphics[width=\textwidth]{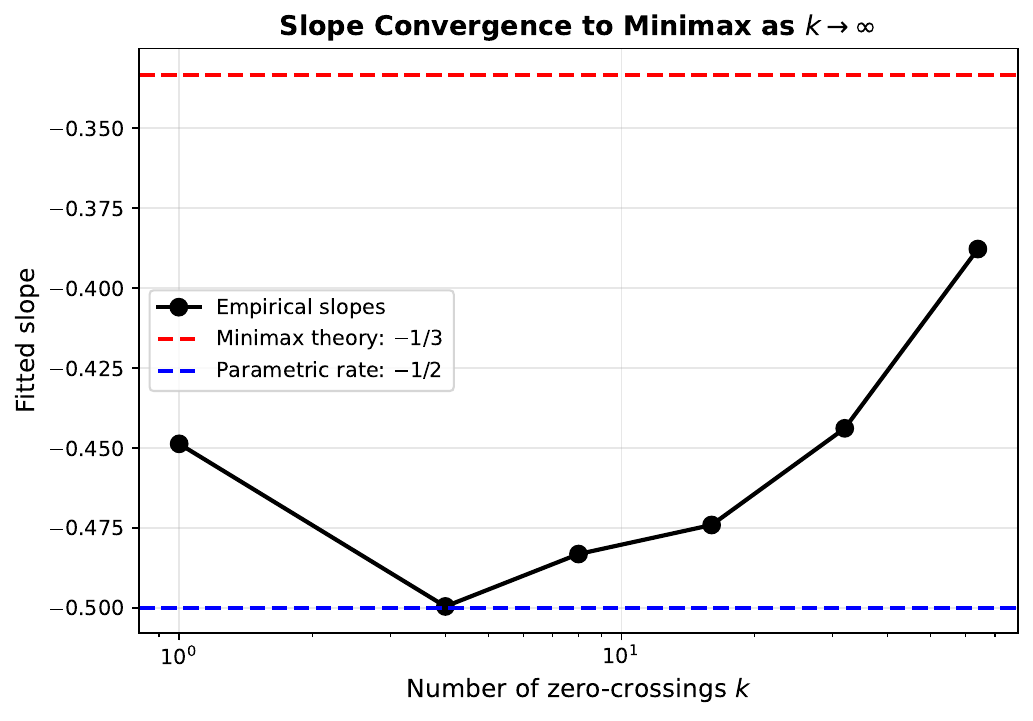}
\caption{Synthetic slope convergence to the minimax $-1/3$ rate as zero-crossings increase.}
\label{fig:slope-convergence}
\end{figure}

\section{Real-Model Experiment Details}
\label{app:real-model}

\paragraph{MMLU experimental protocol.} We query each NIM-hosted model with the standard 4-way multiple-choice prompt (``Answer the following multiple choice question. Reply with ONLY the letter (A, B, C, or D)''). For each question, we extract the logprobs of the four answer tokens A, B, C, D at the first generated position, apply softmax over those four logits, and take the maximum as the model's confidence. The model's predicted answer is the argmax. Records with API errors (e.g., 403 Forbidden, content-filter rejection) are excluded. Each model run uses a single deterministic temperature-0 generation per question.

\paragraph{Confidence saturation (full data, $N = 14{,}042$ each).} On the complete MMLU set, max-softmax confidence is heavily concentrated near 1: Llama-3.1-405B has mean confidence 0.96 with 77\% of questions saturated ($> 0.99$); Qwen3-Next-80B has mean 0.98 with 86\% saturated; Llama-4-Maverick-17B has mean 0.99 with 92\% saturated. The 405B model has the most spread (lowest saturation rate), as expected for a larger and more capable model. Saturation affects the resolution of high-confidence bins but not the convergence rate, which is governed by the optimal bin count $B^*$.

\paragraph{Lipschitz constant estimation.} For each model, we partition the score range into 20 equal-width bins, compute the empirical accuracy in each bin with at least 30 samples, take the calibration gap $\hat\calmap_b = \mathrm{acc}_b - p_b$ at each bin center, and form the adjacent-bin slope $|\hat\calmap_{b+1} - \hat\calmap_b|/|p_{b+1} - p_b|$. We report the 75th percentile of these slopes, capped at $\hat\Lip = 5$, as a robust estimate. This is sensitive to bin choice and noise; we recommend the audit floor be reported with $\hat\Lip$ from at least two binning resolutions.

\paragraph{Pseudo-classifier control experiment.} As a sanity check before running the real-model experiment, we constructed a 10-class softmax classifier from Gaussian logits with $\errrate \approx 0.10$ ($N = 20{,}000$, full-dataset ECE $\approx 0.31$). Subsampling at $\nsamples \in \{50, \ldots, 10^4\}$ with optimal $B^*$ bins (200 replicates per $\nsamples$) confirms that (i) the verification floor (red dashed in Figure~\ref{fig:pseudo-real-model-error}) tracks the empirical std of estimates across the full range, and (ii) the phase transition near $\nsamples \cdot \errrate \approx 1$ is visible (Figure~\ref{fig:pseudo-real-model-phase}). The pseudo-classifier results match the MMLU real-model findings of Section~\ref{sec:empirics}.

\begin{figure}[h]
\centering
\includegraphics[width=\textwidth]{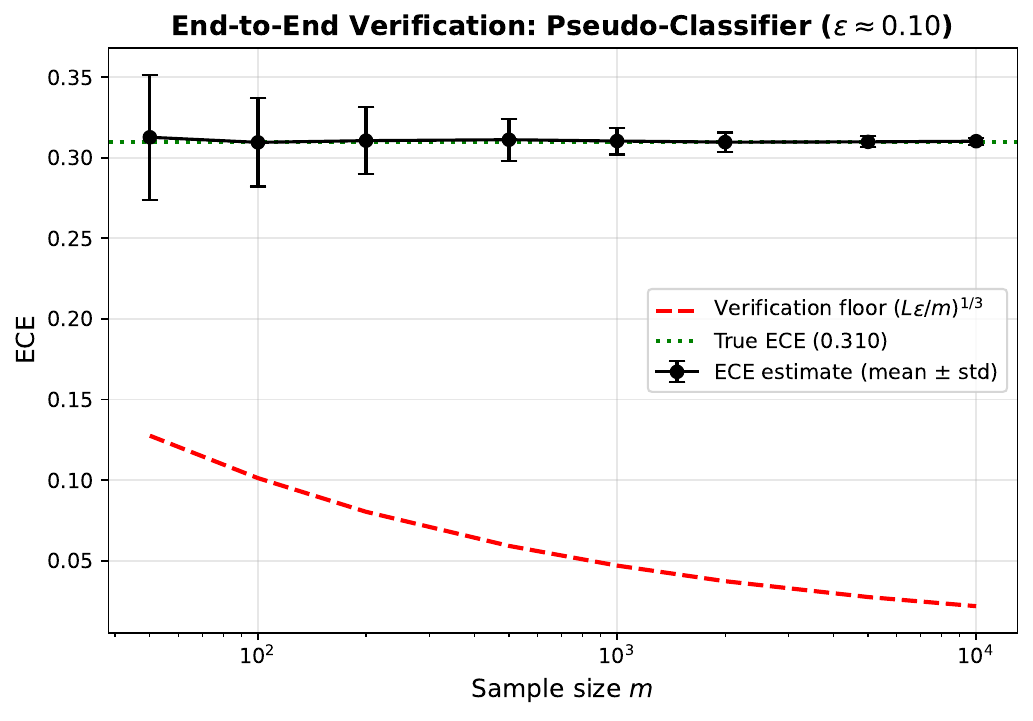}
\caption{Pseudo-classifier control: ECE estimate $\pm$ std versus $\nsamples$, with the verification floor overlaid.}
\label{fig:pseudo-real-model-error}
\end{figure}

\begin{figure}[h]
\centering
\includegraphics[width=\textwidth]{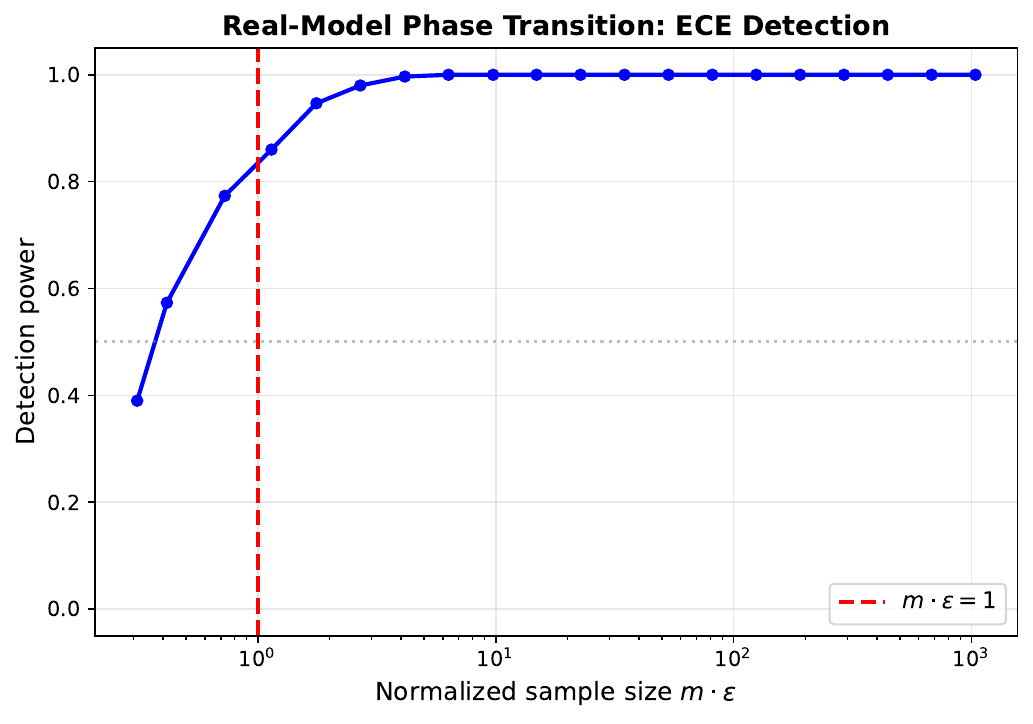}
\caption{Pseudo-classifier control: detection power versus $\nsamples \cdot \errrate$, with the predicted phase transition at $\nsamples \cdot \errrate = 1$.}
\label{fig:pseudo-real-model-phase}
\end{figure}

\subsection{Named Model Comparisons}
\label{app:named-models}

Full methodology: pairwise accuracy gaps between 10 frontier models on 4 benchmarks using published scores from technical reports and model cards. Verification floor: $\targetacc_{\mathrm{acc}} = 2\sqrt{\errrate(1-\errrate)/n}$. The main-text Table~\ref{tab:named} shows the top comparisons; full data (64 pairs) is in the supplementary materials.

\section{Full Proof of Theorem~\ref{thm:compositional} (Compositional Verification)}
\label{app:compositional}

\textbf{Lipschitz composition.} By the chain rule, if $g_k$ is $L_k$-Lipschitz in its first argument, then $g_\recalround \circ \cdots \circ g_1$ is $(\prod_{k=1}^{\recalround} L_k)$-Lipschitz. Since $\calmap_{\mathrm{sys}}(p) = \calmapraw_{\mathrm{sys}}(p) - p$ and $p \mapsto p$ is 1-Lipschitz, $\Lsystem \leq \prod_k L_k + 1$.

\textbf{Verification cost.} Applying Theorem~\ref{thm:upper-bound} with $\Lip = \Lsystem$: $\nsamples_{\mathrm{sys}} = \Theta(\Lsystem \errrate/\targetacc^3)$. For homogeneous $L_k = \Lip$: $\Lsystem \approx \Lip^\recalround$, giving exponential scaling. $\square$

\textbf{Agent loop bound.} For an agent with $\recalround$ reasoning iterations and $\Lip = 2$: $\Lip^{10} = 1024$, $\Lip^{20} > 10^6$.

\textbf{Caveats.} The Lipschitz composition model is optimistic in two directions: discontinuous components (e.g., discrete retrieval) yield $\Lcomp = \infty$, making end-to-end verification impossible at \emph{any} sample size; conversely, contractive components ($L_k < 1$) reduce the effective $\Lsystem$. Empirical estimation of $\Lsystem$ via finite-difference approximation is recommended before applying the exponential bound. Independence of per-component calibration gaps is also assumed; correlated miscalibration can either help or hurt depending on structure.

\textbf{Worked example.} A two-stage retriever--classifier with $\Lip = 2$ each gives $\Lsystem \leq 5$, requiring $5\times$ the data of a single classifier. With $\Lip = 3$: $10\times$. A three-stage pipeline at $\Lip = 2$: $9\times$; at $\Lip = 3$: $28\times$.

\end{document}